%% file: Meta_Learning_in_the_Continuous_Time_Limit.tex
\documentclass[11pt]{article}
\usepackage[paper=letterpaper,margin=1in]{geometry}

\usepackage{amsfonts}
\usepackage{amsmath}
\usepackage{amsthm}
\usepackage{amssymb}
\usepackage{mathrsfs}
\usepackage{bm}
\usepackage{mathtools}
\usepackage{fancyhdr}
\usepackage{algorithm}
\usepackage{algpseudocode}
\usepackage[frak=esstix]{mathalfa}
\usepackage{old-arrows}
\usepackage[parfill]{parskip}
\usepackage{xspace}
\usepackage{xfrac}
\usepackage{afterpage}
\usepackage{framed}
\usepackage[inline]{enumitem}
\usepackage{lmodern}

\usepackage[numbers]{natbib}
\usepackage{upgreek}
\usepackage{subcaption}

\usepackage[svgnames]{xcolor}
    \definecolor{notecolor}{RGB}{137,89,168}
    \definecolor{quotecolor}{RGB}{66,113,174}
    \definecolor{warningcolor}{RGB}{249,145,87}
\usepackage[colorlinks]{hyperref}
	\hypersetup{linkcolor=[RGB]{200,40,41}}
	\hypersetup{citecolor=[RGB]{66,113,174}}
	\hypersetup{urlcolor=[RGB]{231,197,71}}
\usepackage[all]{hypcap}
\usepackage{seqsplit}
\usepackage{xstring}

\usepackage{accents}
\usepackage{trimclip}

\usepackage{tikz}
    \usetikzlibrary{
        calc,
        cd,
        }

\usepackage{graphicx}
\usepackage[colorinlistoftodos]{todonotes}

\usepackage[short]{optidef}
\usepackage[capitalize]{cleveref}

\usepackage{etoolbox}

\makeatletter
\patchcmd\deferred@thm@head
  {\addvspace{-\parskip}}
  {}
  {}{\typeout{\string\deferred@thm@head patch failed!}}
\makeatletter

\makeatletter
\newcommand\RedeclareMathOperator{%
  \@ifstar{\def\rmo@s{m}\rmo@redeclare}{\def\rmo@s{o}\rmo@redeclare}%
}
\newcommand\rmo@redeclare[2]{%
  \begingroup \escapechar\m@ne\xdef\@gtempa{{\string#1}}\endgroup
  \expandafter\@ifundefined\@gtempa
     {\@latex@error{\noexpand#1undefined}\@ehc}%
     \relax
  \expandafter\rmo@declmathop\rmo@s{#1}{#2}}
\newcommand\rmo@declmathop[3]{%
  \DeclareRobustCommand{#2}{\qopname\newmcodes@#1{#3}}%
}
\@onlypreamble\RedeclareMathOperator
\makeatother

\newcommand{\ie}{\textit{i.e.}}

\newcommand{\aka}{\textit{a.k.a.}\xspace}

\newcommand\transpose{^{\mathpalette\raiseT{\scriptstyle\intercal}}}
\newcommand\raiseT[2]{\raisebox{0.2ex}{$#1#2$}}

\newcommand{\inner}[2]{\left\langle{#1},{#2}\right\rangle}

\DeclareMathOperator{\expect}{\mathbb{E}}

\newcommand{\Rd}{{\mathbb{R}^{d}}}

\DeclareMathOperator{\Hess}{Hess}

\newcommand{\eps}{\varepsilon}

\newcommand{\bnorm}[1]{\left\|{#1}\right\|}

\newcommand{\lambdamin}{\lambda_\textnormal{min}}

\newcommand{\kth}{$k$-th }

\DeclareMathOperator*{\st}{\,:\,}

\newcommand*\mcupinn[2]{\vcenter{\hbox{$\mathsurround=0pt
  \ifx\displaystyle#1\textstyle\else#1\fi\bigcup$}}}

\newcommand*\mcapinn[2]{\vcenter{\hbox{$\mathsurround=0pt
  \ifx\displaystyle#1\textstyle\else#1\fi\bigcap$}}}

\DeclarePairedDelimiter{\norm}{\lVert}{\rVert}

\DeclareMathOperator*{\argmin}{arg\,min}

\theoremstyle{plain}
\newtheorem{theorem}{Theorem}[section]
\newtheorem{lemma}[theorem]{Lemma}
\newtheorem{corollary}[theorem]{Corollary}

\newtheorem{proposition}[theorem]{Proposition}

\theoremstyle{definition}

\newtheorem{assumption}[theorem]{Assumption}

\newcommand{\defeq}{\coloneqq}

\def\*#1{\bm{#1}}

\makeatletter
\@tfor\next:=abcdefghijklmnopqrstuvwxyzABCDEFGHIJKLMNOPQRSTUVWXYZ\do{%
  \def\command@factory#1{%
    \expandafter\def\csname v#1\endcsname{{\bm{#1}}}
  }
 \expandafter\command@factory\next
}
\makeatother

\makeatletter
\@tfor\next:=abcdefghijklmnopqrstuvwxyzABCDEFGHIJKLMNOPQRSTUVWXYZ\do{%
  \def\command@factory#1{%
    \expandafter\def\csname bar#1\endcsname{{\overline{#1}}}
  }
 \expandafter\command@factory\next
}
\makeatother

\makeatletter
\@tfor\next:=abcdefghijklmnopqrstuvwxyzABCDEFGHIJKLMNOPQRSTUVWXYZ\do{%
  \def\command@factory#1{%
    \expandafter\def\csname vbar#1\endcsname{{\overline{\bm{#1}}}}
  }
 \expandafter\command@factory\next
}
\makeatother

\makeatletter
\@tfor\next:=ABCDEFGHIJKLMNOPQRSTUVWXYZ\do{%
  \def\command@factory#1{%
    \expandafter\def\csname #1#1\endcsname{{\mathbb{#1}}}
  }
 \expandafter\command@factory\next
}
\makeatother

\makeatletter
\@tfor\next:=ABCDEFGHIJKLMNOPQRSTUVWXYZ\do{%
  \def\command@factory#1{%
    \expandafter\def\csname cal#1\endcsname{{\mathcal{#1}}}
  }
 \expandafter\command@factory\next
}
\makeatother

\makeatletter
\@tfor\next:=ABCDEFGHIJKLMNOPQRSTUVWXYZ\do{%
  \def\command@factory#1{%
    \expandafter\def\csname vcal#1\endcsname{\bm{\mathcal{#1}}}
  }
 \expandafter\command@factory\next
}
\makeatother

\makeatletter
\def\greekvectors#1{%
 \@for\next:=#1\do{%
    \def\X##1;{%
     \expandafter\def\csname v##1\endcsname{{\bm{\csname##1\endcsname}}}
     }
   \expandafter\X\next;
  }
}
\greekvectors{Alpha,Beta,Gamma,Delta,Epsilon,Zeta,Eta,Theta,Iota,Kappa,Lambda,Mu,Nu,Xi,Omicron,Pi,Rho,Sigma,Tau,Upsilon,Phi,Chi,Psi,Omega,alpha,beta,gamma,delta,epsilon,zeta,eta,theta,iota,kappa,lambda,mu,nu,xi,omicron,pi,rho,sigma,tau,upsilon,phi,chi,psi,omega,varphi,varepsilon}
\makeatother

\newcommand{\pnt}[1]{{\scriptstyle#1}}

\makeatletter
\@tfor\next:=ABCDEFGHIJKLMNOPQRSTUVWXYZ\do{%
  \def\command@factory#1{%
    \expandafter\def\csname sc#1\endcsname{{\pnt{#1}}}
  }
 \expandafter\command@factory\next
}
\makeatother

\numberwithin{equation}{section}

\newcommand{\ReLu}{\mathrm{ReLu}}

\DeclareMathOperator{\crit}{Crit}

\newcommand{\shorthand}[1]{\textsc{#1}}
\newcommand{\maml}{\shorthand{maml}\xspace}
\newcommand{\fomaml}{\shorthand{fo-maml}\xspace}
\newcommand{\ode}{\shorthand{ode}\xspace}
\newcommand{\odes}{\shorthand{ode}s\xspace}
\newcommand{\ivp}{\shorthand{ivp}\xspace}

\newcommand{\bmm}{\shorthand{bi-maml}\xspace}
\newcommand{\svm}{\shorthand{svm}\xspace}

\ifundef{\abstract}{}{\patchcmd{\abstract}%
    {\quotation}{\quotation\noindent\ignorespaces}{}{}}

\crefname{assumption}{Assumption}{Assumptions}

\title{Meta Learning in the Continuous Time Limit}

\author{Ruitu Xu%
\thanks{Department of Statistics and Data Science, Yale University. E-mail: \texttt{ruitu.xu@yale.edu}.} \and
Lin Chen%
\thanks{Department of Electrical Engineering, Yale University. E-mail: \texttt{lin.chen@yale.edu}.} \and
Amin Karbasi%
\thanks{Department of Electrical Engineering \& Computer Science, Yale University. E-mail: \texttt{amin.karbasi@yale.edu}.}
}
\date{}

\begin{document}
    
\maketitle

\begin{abstract}
    In this paper, we establish  the ordinary differential equation (\ode) that underlies the training dynamics of Model-Agnostic Meta-Learning (\maml). Our continuous-time limit view of the process eliminates the influence of the manually chosen step size of gradient descent and includes the existing gradient descent training algorithm as a special case that results from a specific discretization. We show that the \maml \ode enjoys a linear convergence rate to an approximate stationary point of the \maml loss function for strongly convex task losses, even when the corresponding \maml loss is non-convex. Moreover, through the analysis of the \maml \ode, we propose a new \bmm training algorithm that significantly reduces the computational burden associated with existing \maml training methods.
    To complement our theoretical findings, we perform empirical experiments to showcase the superiority of our proposed methods with respect to the existing work.
\end{abstract}

\input{introduction}

\section{Preliminaries}\label{sec:prelim}

Before turning to the discussion of the continuous-time limit of \maml, we briefly introduce a widely-used approach for taking the continuous-time limit of discrete-time algorithms and the approach we use later for its analysis.

\paragraph{Optimization through lens of ODE.}\label{par:ode}

There is an extensive literature on the topic of understanding discrete-time algorithms through the lens of \odes~\citep{schropp2000dynamical,helmke2012optimization,lu2020s}, and recent developments in this field offer novel perspectives for looking at discrete-time optimization algorithms~\citep{su2014differential,muehlebach2019dynamical}. For example, \citet{shi2019acceleration} developed a first-order optimization algorithm by performing discretization on \odes that correspond to Nesterov's accelerated gradient descent. \citet{Krichene2015} proposed a family of continuous-time dynamics for convex functions where the corresponding solution converges to the optimum value at an optimal rate. However, there can be multiple \odes that correspond to the same discrete-time algorithm, and it oftentimes requires strong mathematical intuitions when it comes to taking the continuous limit. In this paper, we take the most intuitive approach by letting the step size of the gradient descent on the \maml loss $F$ go to zero, and the resulting \ode is a gradient flow on $F$.

\paragraph{Lyapunov's direct method.}

One of the most commonly used approaches for analyzing the convergence of \odes is Lyapunov's direct method~\citep{lyapunov1992general,hirsch2012differential,wilson2016lyapunov}, which is based on constructing a positive definite Lyapunov function $\calE:\Rd\to\RR$ that decreases along the trajectories of the dynamics $\dot w$:
\begin{align}\label{eqn:lyapunov-cond}
    \frac{d}{dt}\calE(w(t)) = \inner{\nabla\calE(w(t))}{\dot w(t)} \leq 0.
\end{align}
This method is a generalization of the idea that measures the ``energy'' in a system, and the existence of such a continuously differentiable Lyapunov function guarantees the convergence of the dynamical system. 

\section{Main Results}\label{sec:main-results}

In this paper, we analyze the \maml algorithm on strongly convex functions with its continuous-time limit and establish a linear convergence rate.
In addition, we propose a new algorithm named \texttt{Biphasic MAML} (\bmm) as an alternative to the original \maml. Unlike \maml where we always minimize the function value on $F$,
the optimization process of \bmm can be divided into two phases. In the first phase, \bmm optimizes the expected task loss $f$ until it reaches its approximate global minimum. In the second phase, it runs \maml until it finds an approximate critical point.
We show that \bmm also enjoys the same  $O(\log\frac{1}{\eps})$ iteration complexity on strongly convex functions. While the iteration complexity of \maml and \bmm share the same order, \bmm has a lower computational complexity, because it performs gradient descent on $f$ rather than $F$ in the first phase and thereby avoids computing the Hessian. In contrast, \maml performs gradient descent on $F$, which involves computing the Hessian of $f$. We will elaborate further later in this section. Moreover, we will present empirical results in \cref{sec:experiments} to show that \bmm significantly outperforms vanilla \maml.

Both our new analyses on \maml and the new \bmm algorithm are based on our analysis of the landscape of $F$. \cref{fig:counterexample} illustrates an example where $F$ is indeed non-convex. This example has two tasks where we take $f_1(w) = 0.505 w^2 - \sin(w)$ and $f_2(w) = 0.505 w^2 - 0.0001\sin(100w)$. Both functions are $0.01$-strongly convex and $2.01$-smooth. We observe that $F'(w)$ is not monotone increasing and that $F''(w)$ is not always positive. These imply that $F$ is non-convex.
While $F$ is non-convex in general as illustrated, we prove that $F$ has a unique critical point and is strongly convex on a convex set around the critical point, which implies that it is the global minimum. Of course, the function $F$ can be non-convex outside the convex set. 

\begin{figure}[htb]
\centering
\begin{subfigure}{.4\textwidth}
  \centering
  \includegraphics[width=\linewidth]{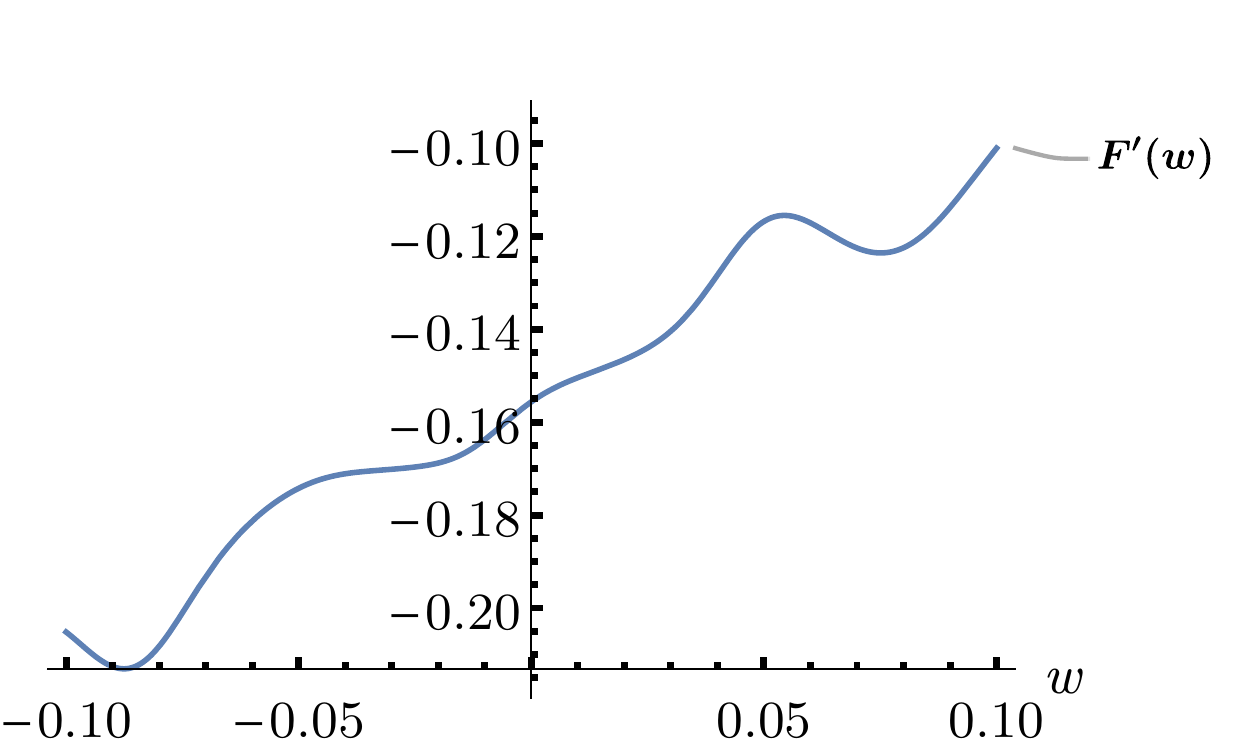}
  \caption{Plot of $F'(w)$.}
  \label{fig:plot-grad-F}
\end{subfigure}\hfil
\begin{subfigure}{.4\textwidth}
  \centering
  \includegraphics[width=\linewidth]{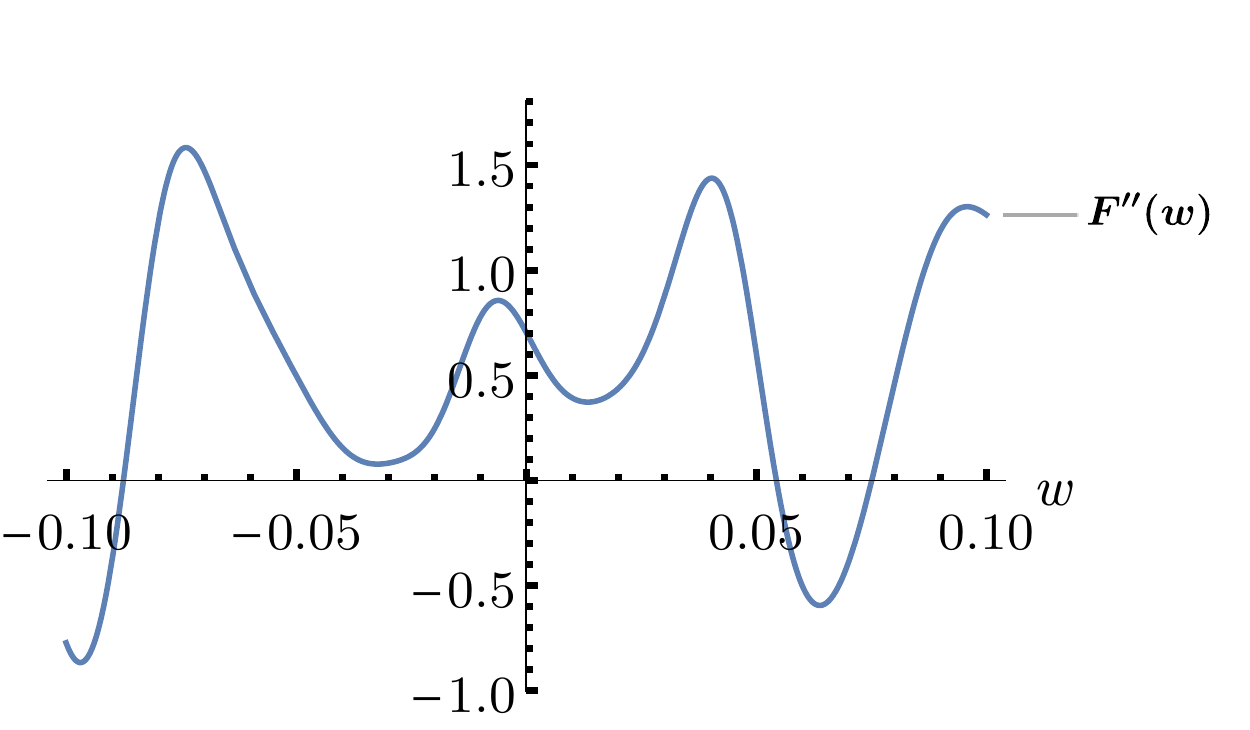}
  \caption{Plot of $F''(w)$.}
  \label{fig:plot-hess-F}
\end{subfigure}
\caption{An example of non-convex \maml loss $F(w)$ even if its corresponding task losses satisfy all \cref{assum:bounded-variance,assum:smoothness,assum:strongly-convex,assum:lipschitz}. Here we let $f_1(w) = 0.505 w^2 - \sin(w)$ and $f_2(w) = 0.505 w^2 - 0.0001\sin(100w)$. Notice that both $f_1''(w)\geq 0.01$ and $f_2''(w)\geq 0.01$ for all $w\in\Rd$. Both functions are $0.01$-strongly convex and $2.01$-smooth. It is not hard to see also that \cref{assum:bounded-variance,assum:lipschitz} are both satisfied for some finite $\sigma$ and $\kappa$. Taking the \maml step size as $\alpha = 0.4$, we have a non-convex \maml loss $F$ with its first- and second-order derivatives as indicated in \cref{fig:plot-grad-F,fig:plot-hess-F}.}
\label{fig:counterexample}
\end{figure}

\subsection{Original and Biphasic MAML (BI-MAML) Algorithms}

In this section, we present the original \maml and \bmm algorithms. We begin with the original \maml. 
In particular, we investigate \maml under a continuous-time limit. Recall that the update rule of \maml on $w$ follows a gradient descent on $F(w)$, \ie,
\begin{align}
    w^+ = w - \beta \nabla F(w) = w - \beta \expect_{i\sim p}\nabla F_i(w) = w - \beta\expect_{i\sim p}[A_i(w)\nabla f_i(w-\alpha\nabla f_i(w))], \label{eqn:maml-update}
\end{align}
where $w$ denotes the iterate input, $w^+$ denotes the iterate output, $\beta$ represents the step size, and $A_i(w) \defeq I_d - \alpha\nabla^2 f_i(w)$ is a shorthand for the Hessian correction term. Here we see that computing the gradient of $F$ requires the Hessian of $f$. As mentioned above, \bmm runs faster because it avoids computing the Hessian of $f$ by performing gradient descent on $f$ in the first phase. 
In this paper we consider \maml where the step size $\alpha$ of the task-specific gradient descent remains constant while the step size $\beta$ for the \maml gradient descent goes to zero. \cref{prop:maml-ode} presents the continuous-time limit for \maml, which we term as the \maml \ode.

\begin{proposition}[Proof in \cref{sec:ode-proof}]\label{prop:maml-ode}
    If the losses $f_i$ are twice differentiable, the continuous-time limit for \maml is
    \begin{align}\label{eqn:ode}
        \dot w & = -\nabla F(w) = -\nabla f(w) + \expect_{i\sim p}[B_i(w)\nabla f_i(w)]
    \end{align}
    where $B_i(w) \defeq \alpha(\nabla^2 f_i(w)+\nabla^2 f_i(\tilde w_i)) - \alpha^2\nabla^2 f_i(w)\nabla^2 f_i(\tilde w_i)$ and $\tilde w_i$ is a convex combination of $w$ and $w-\alpha\nabla f_i(w)$.
\end{proposition}

The first term on the right-hand side of \eqref{eqn:ode} represents a gradient flow on $f$, and the second term that follows is the key term that differentiates a \maml gradient descent on the \maml loss function $F$ from a vanilla gradient descent on the expected loss function $f$.
Due to the compositional nature of the \maml loss $F(w)$, the second-order information is required to evaluate its gradient. Recall that by definition, we have $\nabla F(w) = \expect_{i\sim p}[(I_d - \alpha\nabla^2 f_i(w))\nabla f_i(w-\alpha\nabla f_i(w))]$. To reduce such cost on computing the Hessian, \citet{Finn2017} proposed the first-order model-agnostic meta-learning (\fomaml), which is a first-order approximation of \maml that replaces the second-order term $I_d - \alpha\nabla^2 f_i(w)$ with an identity matrix. However, this may cause a failure in convergence, as mentioned in \citep{fallah2019convergence}. In comparison, our new \bmm algorithm achieves a similar computational efficiency by evaluating only first-order information while at the same time still being able to converge to an approximate stationary point. 

Inspired by our convergence analysis on \maml \ode, we propose a different dynamic with gradient flow in two stages called \bmm \ode. The first stage of \bmm \ode is a gradient flow on the expected loss function $f$ that converges to one of its stationary point, \ie, $\dot w = -\nabla f(w)$. This is followed by the second stage, which is a \maml \ode starting from an approximate stationary point provided by the first stage. 
\maml is a forward Euler integration of its continuous-time limit \eqref{eqn:ode}: a gradient flow on the \maml loss $F$. Analogously, the \bmm algorithm is a forward Euler integration of \bmm \ode: a gradient descent on $f$ followed by a gradient descent of $F$.
More detailed descriptions of \bmm and \bmm \ode are given in \cref{algo:bi-maml} and \cref{algo:bi-maml-ode}.

\begin{figure}[htb]
\begin{minipage}{0.49\textwidth}
\begin{algorithm}[H]
\centering
\caption{Biphasic \maml (\bmm)}\label{algo:bi-maml}
    \begin{algorithmic}[1]
        \Require Loss functions $\{f_i(w)\}_{i\in[M]}$, \maml parameter $\alpha$, step size $\beta$, tolerance level $\varepsilon_0, \varepsilon$.
        \State {\bf initialize} $w(0)\in\Rd$ arbitrarily
        \For{$t\in \NN \cup \{0\}$}
            \If{$\|\nabla f(w(t))\|\leq\varepsilon_0$}
            \State $w(t+1)\gets w(t) - \beta \nabla f(w(t))$
            \Else
            \State $w(t+1) \gets w(t)  - \beta \nabla F(w(t))$
            \EndIf
              \State \Return $w(t+1)$ if $\|\nabla F(w(t))\|\leq\varepsilon$
        \EndFor
    \end{algorithmic}    
\end{algorithm}
\end{minipage}\hfill
\begin{minipage}{0.49\textwidth}
\begin{algorithm}[H]
\centering
\caption{\bmm \ode}\label{algo:bi-maml-ode}
    \begin{algorithmic}[1]
        \Require Loss functions $\{f_i(w)\}_{i\in[M]}$, \maml parameter $\alpha$, step size $\beta$, tolerance level $\varepsilon_0, \varepsilon$.
        \State {\bf initialize} $w(0)\in\Rd$ arbitrarily
        \For{$t\in [0,\infty)$}
            \If{$\|\nabla f(w(t))\|\leq\varepsilon_0$}
            \State $\frac{dw}{dt}\gets  - \beta \nabla f(w(t))$
            \Else
            \State $\frac{dw}{dt} \gets  - \beta \nabla F(w(t))$
            \EndIf
              \State \Return $w(t)$ if $\|\nabla F(w(t))\|\leq\varepsilon$
        \EndFor
    \end{algorithmic}    
\end{algorithm}
\end{minipage}
\end{figure}

\subsection{Summary of Assumptions, Results, and Techniques}
In this section, we establish our theoretical results for \maml and \bmm when the loss functions $f_i(w)$ are strongly convex and smooth with bounded gradient variance among the tasks. Formally, we make the following assumptions. 


\begin{assumption}\label{assum:smoothness}
    For every $i\in [M]$, the loss $f_i(w)$ is twice differentiable and $L_i$-smooth, \ie, for every $w,u\in\Rd$, we have $\norm{\nabla f_i(w) - \nabla f_i(u)} \leq L_i \norm{w-u}$.
\end{assumption}

\begin{assumption}\label{assum:strongly-convex}
    For every $i\in [M]$, the loss $f_i(w)$ is $\mu_i$-strongly convex, \ie, for every $w,u\in\Rd$, there exists positive $\mu$, such that $\norm{\nabla f_i(w) - \nabla f_i(u)} \geq \mu_i \norm{w-u}$.
\end{assumption}

\begin{assumption}\label{assum:bounded-variance}
    For any $w\in\Rd$, the variance of gradient $\nabla f_i(w)$ is bounded, \ie, there exists non-negative $\sigma$, such that $\expect_{i\sim p}\|\nabla f(w) - \nabla f_i(w)\|^2 \leq \sigma^2$.
\end{assumption}

\begin{assumption}\label{assum:lipschitz}
    For every $i\in [M]$, the Hessian for loss $f_i(w)$ is $\kappa_i$-Lipschitz continuous, \ie, for every $w,u\in\Rd$, we have $\norm{\nabla^2 f_i(w) - \nabla^2 f_i(u)} \leq \kappa_i \norm{w-u}$.
\end{assumption}

To simplify the notation, we denote $L \defeq \max_i L_i$, $\mu \defeq \min_i \mu_i$, and $\kappa \defeq \max_i \kappa_i$ in the rest of the paper.
Because $f_i$ is twice differentiable, \cref{assum:smoothness} is equivalent to $-L_iI_d \preceq \nabla^2 f_i(w) \preceq L_iI_d$.
We note that \citet{finn2019online} assumed all the assumptions above, except \cref{assum:bounded-variance} because they considered online meta-learning where functions can be selected in an adversarial manner. On the other hand, they assumed that the functions are Lipschitz \citep[Assumption~1.1]{finn2019online}, which may contradict the strong convexity assumption \citep[Assumption~2]{finn2019online} in their paper. Similarly, \citet{fallah2019convergence} assumed all but \cref{assum:strongly-convex}. We remark that \cref{assum:strongly-convex} implies the boundedness of the \maml loss $F$ from below, but it does not guarantee the convexity of $F$.
See \cref{fig:counterexample} for an example of non-convex \maml loss $F$ with corresponding task losses $f_i$ satisfying \cref{assum:bounded-variance,assum:smoothness,assum:strongly-convex,assum:lipschitz}.
In other words, while $f_i$ are strongly convex, $F$ can still be non-convex. Hence, minimizing $F$ is challenging as we are dealing with a non-convex optimization problem. 

As we mention above, we make an additional assumption (\cref{assum:strongly-convex}), compared to the set of assumptions in \citep{fallah2019convergence}. They showed that the \maml algorithm outputs a solution that guarantees $\|\nabla F(w)\|\le \eps$ in $\calO(\frac{1}{\eps^2})$ iterations. Under this additional assumption, we significantly improve the result in two aspects. First, we show that our proposed algorithm finds a solution such that $\|\nabla F(w)\|\le \eps$ in only $O(\log \frac{1}{\eps})$ iterations (\cref{thm:main}). This is indeed an exponential improvement on \citep{fallah2019convergence} in terms of iteration complexity. Second, we characterize the landscape of $F$. While $F$ is non-convex in general, we prove that its stationary point is also the global minimum (\cref{thm:F-unique-critical-point}). Therefore, the solution returned by our algorithm is close to not only a critical point but also the global minimum. 


Our main results in \cref{thm:main} and \cref{thm:bi-maml-convergence} show that the \maml \ode and \bmm \ode achieve linear convergence in finding a critical point on the \maml loss $F$. 
\begin{theorem}[Iteration complexity, proof in \cref{sec:convergence}]\label{thm:main}
    Suppose the loss function $f_i(w)$ satisfies \cref{assum:smoothness,assum:strongly-convex,assum:bounded-variance,assum:lipschitz}, if 
    \begin{align*}
        \alpha < \min\left\{ \frac{1}{2L}, \frac{\mu^{3/2}}{36\kappa \sigma+28\kappa\sqrt{\mu}\sigma},\frac{\mu^{3/2}}{16\sqrt{L}\kappa \sigma+24\kappa\sqrt{\mu}\sigma}, \sqrt[\leftroot{-2}\uproot{2}3]{\frac{2}{15}}\mu^{1/3}L^{-5/3}, \sqrt{\frac{1}{15}}\mu^{1/2}L^{-2}, \sqrt{\frac{1}{15}}\mu L^{-2} \right\}
    \end{align*}
    then the \maml \ode finds a solution $\widehat w$ such that $\|\nabla F(\widehat w(t))\| \leq \varepsilon$ after at most running for 
    \begin{align*}
        t =\calO\left[ \frac{1}{\mu}\log\left( \frac{(5+\frac{9}{\sqrt{\mu}})(\mu^2\sigma \|\nabla f(w(0))\|^2 - \frac{\mu\sigma^3}{2})}{4\iota\sigma^2\varepsilon} \right) \right],
    \end{align*}
    if $\|\nabla f(w(0))\|^2 > \frac{\sigma^2}{\mu}$ and 
    \begin{align*}
        t = \calO\left[ \frac{16}{\mu} \log\left( \frac{(5+\frac{9}{\sqrt{\mu}})\sigma}{4\varepsilon} \right) \right]
    \end{align*}
    otherwise, where $\iota>0$ is a small constant.
\end{theorem}

\begin{theorem}[Proof in \cref{sec:convergence}]\label{thm:bi-maml-convergence}
    Suppose the loss function $f_i(w)$ satisfies \cref{assum:smoothness,assum:strongly-convex,assum:bounded-variance,assum:lipschitz} and $\varepsilon_0$ is the tolerate level set in \cref{algo:bi-maml-ode}, if 
    \begin{align*}
        \alpha < \min\left\{ \frac{1}{2L}, \frac{\mu}{36\kappa \varepsilon_0+28\kappa\sigma},\frac{\mu^{3/2}}{16\sqrt{L}\kappa \sigma+24\kappa\sqrt{\mu}\sigma} \right\}
    \end{align*}
    then the \bmm \ode finds a solution $\widehat w$ such that $\|\nabla F(\widehat w(t))\| \leq \varepsilon$ after at most running for 
    \begin{align*}
        t = \frac{1}{\mu}\calO\left[ \log\left( \frac{(9\varepsilon_0 + 5 \sigma)\|\nabla f(w(0))\|}{4\varepsilon_0\varepsilon} \right) \right].
    \end{align*}
\end{theorem}

\cref{thm:main} says that if the \maml parameter $\alpha$ is small enough, then the \maml \ode finds an approximate stationary point of the \maml loss $F$ in $\calO(\log\frac{1}{\varepsilon})$ time. This approximate stationary point is at the same time an approximate global minimum, as implied by \cref{thm:F-unique-critical-point} later. Similarly, \cref{thm:bi-maml-convergence} states that whenever the \maml parameter $\alpha$ is small enough so that $F$ is strongly convex for every point $w$ such that $\|\nabla f(w)\|\leq \varepsilon_0$, the \bmm \ode finds an approximate global minimum of $F$ in $\calO(\log\frac{1}{\varepsilon})$ time.

We prove \cref{thm:main} and \cref{thm:bi-maml-convergence} with a two-phase analysis where the transition between two phases depends on the norm of $\nabla F(w)$. In the first phase we conduct a Lyapunov function analysis on the Lyapunov candidate function $\|\nabla f(w)\|^2$ and under the dynamics defined by the \odes. The second phase follows as a landscape characterization of the \maml loss $F$. Intuitively, $\|\nabla F(w)\|$ can be large at initialization and will become smaller over the course of the gradient flow. We note that $\|\nabla F(w)\|$ is close to $\|\nabla f(w)\|$ if $\alpha$ is small due to the fact that $\nabla F(w) = (I_d - \alpha\nabla^2 f_i(w))\nabla f_i(w - \alpha\nabla f_i(w))$, and we therefore choose to analyze the gradient norm on $f$ instead of $F$. When $\|\nabla f(w)\|$ is large and we are far from the stationary point, the analysis on the Lyapunov function $\calE(w(t)) = \|\nabla f(w(t))\|^2$ helps establish a linear convergence rate for $\|\nabla f(w)\|$ to be in the order of $\calO(\sigma)$. This is due to the proof technique shown in \cref{lem:upper-bound-dynamics}.
    
\begin{lemma}[Proof in \cref{sec:convergence}]\label{lem:upper-bound-dynamics}
    Suppose the loss functions $f_i(w)$ satisfy \cref{assum:smoothness,assum:strongly-convex}, then it holds 
    \begin{align}\label{ineqn:dynamics-upper-bound}
        \frac{d}{dt}\frac{1}{2}\bnorm{\nabla f(w)}^2 \leq -\Big( \mu - \frac{5}{4} L^2\alpha(L^3\alpha^2+2L^2\alpha+2) \Big) \|\nabla f(w)\|^2 + \frac{\sigma^2}{2}.
    \end{align}
\end{lemma}
    
When the gradient norm $\|\nabla f(w)\|$ is small enough, $\|\nabla F(w)\|$ is also controlled. Then we enter the second phase, in which we follow a gradient flow inside a convex set where the \maml loss is strongly convex. This establishes another linear convergence rate from $\calO(\sigma)$ down to $\varepsilon$. Combining the above two phases gives us the overall linear rate. In the following subsections, we will explain these two phases in more detail with an emphasis on the proof of \cref{thm:main}. The proof for \cref{thm:bi-maml-convergence} is similar, cf.~\cref{sec:convergence}.

\subsection{Large Gradient Phase: Linear Convergence via Lyapunov Analysis}

When $\|\nabla F(w)\|$ is large, its behavior under the \maml \ode \eqref{eqn:ode}, namely $\frac{d}{dt}\|\nabla F(w(t))\|$,  is not easily tractable due to the non-convex nature of the \maml loss $F$. Hence, we consider instead a Lyapunov candidate function $\calE(w(t)) = \|\nabla f(w(t))\|^2$ in the first phase, where
\begin{align}
    \begin{split}\label{eqn:grad-f-dynamics}
        \frac{d}{dt}\calE(w(t)) & = \nabla f(w)\transpose\nabla^2 f(w) \dot w \\
        & = -\nabla f(w)\transpose\nabla^2 f(w)\nabla f(w) + \nabla f(w)\transpose\nabla^2 f(w)\expect_{i\sim p}[B_i(w)\nabla f_i(w)].
    \end{split}
\end{align}
Even though we are primarily interested in the convergence of $\|\nabla F(w(t))\|$, the convergence analysis of $\|\nabla f(w(t))\|$ is still helpful. We show that when $\alpha$ is small, an upper bound on $\|\nabla f(w)\|$ gives an upper bound on $\|\nabla F(w)\|$, and \emph{vice versa}. Hence we are able to keep track of an upper bound on $\|\nabla F(w(t))\|$ while only having one on $\|\nabla f(w(t))\|$.
However, we need to remark that the convergence of $\|\nabla f(w)\|$ to zero does not imply the convergence of $\|\nabla F(w)\|$. In fact, the global minimum of $f$ may not even be a stationary point of $F$. 

Note that the Lyapunov candidate function $\|\nabla f(w)\|^2$ turns into a true Lyapunov function for the dynamic \eqref{eqn:ode} if $\frac{d}{dt}\bnorm{\nabla f(w)}^2 \leq 0$ for every $w\in\Rd$. To get a linear rate on $\|\nabla f(w)\|^2$, we need to characterize the right-hand side of \eqref{eqn:grad-f-dynamics}. Notice that the first term is a quadratic form of $\nabla f(w)$, while the second term is less tractable due to the expectation. We build an upper bound in \cref{lem:upper-bound} to tackle this second term, and it is achieved through ``pulling out'' the integrand $B_i(w)$ from the expectation and forming a quadratic form that is more friendly to the spectral analysis to follow. \cref{lem:upper-bound} then leads to a more tractable upper bound for the right-hand side of \eqref{eqn:grad-f-dynamics}, as illustrated in \cref{lem:upper-bound-dynamics}. If we replace the inequality in \eqref{ineqn:dynamics-upper-bound} with an equality, the resulting \ode on $\|\nabla f(w)\|^2$ is subject to a closed-form solution, which converges to a constant smaller than $\sigma^2/4$ when $\alpha$ is small. This solution serves as an upper bound on $\|\nabla f(w(t))\|^2$ for any $t\geq 0$.
By making sure that the upper bound in \eqref{ineqn:dynamics-upper-bound} is strictly less than zero, it enables us to provide sufficient conditions on the \maml step size $\alpha$ so that the Lyapunov function is convergent in linear rate to a small constant, as explained in \cref{thm:spectral-bound}.

However, the upper bound on the Lyapunov function $\|\nabla f(w)\|^2$ does not converge to zero, we can only guarantee in phase one that $\|\nabla F(w)\|$ goes below a constant. This issue will be resolved in phase two of the analysis.

%
\subsection{Small Gradient Phase: Unique Global Minimum via Landscape Analysis }\label{par:attraction}

Due to the aforementioned limitations of the Lyapunov method, we propose a landscape analysis that complements the above argument and guarantees the linear rate of the \maml \ode when the gradient norm $\|\nabla F(w)\|$ on the \maml loss is small enough. Recall that the \maml \ode is a gradient flow on $F$, thus the landscape of $F$ determines the behavior of the \maml \ode. If a function is strongly convex, then its gradient flow converges linearly to its unique minimizer. Even though the \maml loss $F$ is not convex in general, we are able to show in \cref{thm:strongly-convex-F} that for any point $w\in\Rd$ with a bounded gradient norm $\|\nabla F(w)\|$, the \maml loss $F$ is both smooth and strongly convex in its neighborhood. This provides us with a powerful tool that enables us to show the global convergence of the \maml \ode, as indicated in \cref{thm:F-unique-critical-point}.

\begin{theorem}[Proof in \cref{sec:strong-convexity}]\label{thm:strongly-convex-F}
    Suppose $f_i(w)$ satisfies \cref{assum:smoothness,assum:strongly-convex}. Then for any $\alpha \leq \min\{\frac{1}{2L},\frac{\mu}{8\kappa(2K + \sigma)}\}$ and
      $ w \in U(K) \defeq \{ w\in \Rd: \|\nabla F(w)\| \leq K \}$, we have $ \frac{\mu}{8}I_d \preceq \Hess(F(w)) \preceq \frac{9L}{8}I_d$, where $\Hess(F(w))$ is the Hessian matrix of $F$ at point $w$.
\end{theorem}

\begin{theorem}[Unique global minimum, proof in \cref{sec:proof-F-unique-critical-point}]\label{thm:F-unique-critical-point}
    If $K> \left(1+ \sqrt{\frac{L}{\mu }}\right)\sigma$ and $\alpha \leq \min\{\frac{1}{4L},\frac{\mu}{8\kappa(2K + \sigma)}\}$, 
    the function $F$ is strongly convex on the convex set $V\left(\frac{(K-\sigma )^2}{2 L}\right) \defeq \{w\in \RR^d: f(w)\le \min_{w'\in \RR^d} f(w') + \frac{(K-\sigma )^2}{2 L} \}$. Furthermore, the set $V\left(\frac{(K-\sigma )^2}{2 L}\right)$ contains the unique critical point of $F$. 
\end{theorem}

We remark that even though \cref{thm:F-unique-critical-point} concludes that the \maml loss $F$ is strongly convex within $V$, $F$ can be non-convex outside $V$ (recall our example in \cref{fig:counterexample}).
Being a sublevel set of a strongly convex function $f$, the set $V$ is convex. Moreover, it is also closed and bounded. Again, by the strong convexity of $f$, its minimum $\min_{w'\in \RR^d} f(w')$ exists and is finite. 
\cref{thm:F-unique-critical-point} implies that 
     there is no critical point outside $V$. Since $F$ is strongly convex within $V$, the unique critical point inside the convex sublevel set $V$ is consequently the global minimizer of $F$. 

\input{experiments}

\section{Conclusions}\label{sec:conclusions}

In this paper we analyze the \maml \ode, a continuous-time limit of \maml, and establish a linear convergence rate to the global minimum of the \maml loss function for strongly convex task losses. We also propose a computationally efficient algorithm \bmm where its continuous-time limit \bmm \ode has the same linear convergence guarantee under milder conditions. We experimentally show that the \bmm method outperforms \maml in a variety of learning tasks.

\section*{Acknowledgements}
We would like to thank Marko Mitrovic for his help in preparation of the paper.



\bibliographystyle{plainnat}
\bibliography{References,lin-ref}

\newpage

\appendix

\section{Proof of \cref{prop:maml-ode}}\label{sec:ode-proof}

\begin{proof}
    Recall the \maml algorithm with update \cref{eqn:maml-update}, \ie,
    \begin{align*}
        w^+ = w - \beta\nabla F(w),
    \end{align*}
    and that $\nabla F_i(w) = (I_d - \alpha \nabla^2 f_i(w))\nabla f_i(w - \alpha\nabla f_i(w))$. Expand the terms to get 
    \begin{align*}
        \nabla F(w) & = \expect_{i\sim p}[(I_d - \alpha \nabla^2 f_i(w))\nabla f_i(w - \alpha\nabla f_i(w))] \\
        & = \expect_{i\sim p}\nabla f_i(w - \alpha\nabla f_i(w)) - \alpha\expect_{i\sim p}\nabla^2 f_i(w)\nabla f_i(w - \alpha\nabla f_i(w)) \\
        & = \expect_{i\sim p}(I_d - \alpha\nabla^2 f_i(\tilde w_i))\nabla f_i(w) - \alpha\expect_{i\sim p}\nabla^2 f_i(w)(I_d - \alpha\nabla^2 f_i(\tilde w_i))\nabla f_i(w) \\
        & = \expect_{i\sim p}(I_d - \alpha\nabla^2 f_i(w))(I_d - \alpha\nabla^2 f_i(\tilde w_i))\nabla f_i(w) \\
        & = \expect_{i\sim p} A_i(w)A_i(\tilde w_i)\nabla f_i(w),
    \end{align*}
    where the first equality follows from definition, the third equality follows from mean value theorem. Here $\tilde w_i$ is a value between $w$ and $w-\alpha\nabla f_i(w)$ such that mean value theorem holds.
    The formula can be further recast into
    \begin{align*}
        \nabla F(w) & = \expect_{i\sim p}[\nabla f_i(w) - \alpha(\nabla^2 f_i(w)+\nabla^2 f_i(\tilde w_i))\nabla f_i(w) + \alpha^2\nabla^2 f_i(w)\nabla^2 f_i(\tilde w_i)\nabla f_i(w)] \\
        & = \nabla f(w) - \expect_{i\sim p}[(\alpha(\nabla^2 f_i(w)+\nabla^2 f_i(\tilde w_i)) - \alpha^2\nabla^2 f_i(w)\nabla^2 f_i(\tilde w_i))\nabla f_i(w)].
    \end{align*}
    If we think of the infinitesimal step size $\beta\to 0$, we obtain an \ode that represents the gradient flow on $F(w)$:
    \begin{align*}
        \dot w & = -\nabla F(w) \\
        & = -\nabla f(w) + \expect_{i\sim p}[\underbrace{(\alpha(\nabla^2 f_i(w)+\nabla^2 f_i(\tilde w_i)) - \alpha^2\nabla^2 f_i(w)\nabla^2 f_i(\tilde w_i))}_{B_i(w)}\nabla f_i(w)].
    \end{align*}
    We define a shorthand $B_i(w)$ for notational convenience.
\end{proof}

\section{Proof of the Convergent Upper Bound}\label{sec:convergence}

\begin{lemma}\label{lem:upper-bound}
    If the loss function $f_i(w)$ satisfies \cref{assum:smoothness,assum:strongly-convex} and $\alpha < \frac{1}{2L}$, then it holds that 
    \begin{align}\label{ineqn:spectral-bound2}
        \nabla f(w)\transpose\nabla^2 f(w)\expect_{i\sim p}[B_i(w) \nabla f_i(w)] \leq \frac{5}{4} L^2\alpha(L^3\alpha^2+2L^2\alpha+2)\|\nabla f(w)\|^2 + \frac{\sigma^2}{2}.
    \end{align}
\end{lemma}

\begin{proof}
    Another upper bound for the third term on the right-hand side of \cref{eqn:grad-f-dynamics} can be derived through relaxing its difference with the quadratic form
    \begin{align*}
        & \nabla f(w)\transpose\nabla^2 f(w)\expect_{i\sim p}[B_i(w) \nabla f_i(w)] - \nabla f(w)\transpose\nabla^2 f(w)\expect_{i\sim p}[B_i(w)] \nabla f(w) \\
        =\ & \expect_{i\sim p}[\nabla f(w)\transpose\nabla^2 f(w)B_i(w)\left( \nabla f_i(w) - \nabla f(w) \right)] \\
        \leq\ & \frac{1}{2}\expect_{i\sim p}\|B_i(w)\transpose\nabla^2 f(w)\nabla f(w)\|^2 + \frac{1}{2}\expect_{i\sim p}\|\nabla f_i(w) - \nabla f(w)\|^2,
    \end{align*}
    where the last inequality follows from Young's inequality.
    This provides yet another upper bound after rearranging the terms as follows:
    \begin{align*}
        \nabla f(w)\transpose\nabla^2 f(w)\expect_{i\sim p}[B_i(w) \nabla f_i(w)] & \leq \nabla f(w)\transpose\nabla^2 f(w)\expect_{i\sim p}[B_i(w)] \nabla f(w) \\
        & \qquad + \frac{L^2}{2}\max_i\|B_i(w)\|^2\|\nabla f(w)\|^2 + \frac{\sigma^2}{2} \\
        & \leq \Big( L\max_i\|B_i(w)\| + \frac{L^2}{2}\max_i\|B_i(w)\|^2 \Big) \|\nabla f(w)\|^2 + \frac{\sigma^2}{2}.
    \end{align*}
    The first and second inequality are due to \cref{assum:bounded-variance,assum:smoothness}. Recall that
    \begin{align*}
        B_i(w) = \alpha(\nabla^2 f_i(w)+\nabla^2 f_i(\tilde w_i)) - \alpha^2\nabla^2 f_i(w)\nabla^2 f_i(\tilde w_i),
    \end{align*}
    and it is not hard to see that $\max_i\|B_i(w)\| \leq 2\alpha L + \alpha^2 L^2$. Hence we conclude that
    \begin{align*}
        \nabla f(w)\transpose\nabla^2 f(w)\expect_{i\sim p}[B_i(w) \nabla f_i(w)] & \leq \Big( L\max_i\|B_i(w)\| + \frac{L^2}{2}\max_i\|B_i(w)\|^2 \Big) \|\nabla f(w)\|^2 + \frac{\sigma^2}{2} \\
        & \leq \frac{1}{2} L^2\alpha(L\alpha+2)(L^3\alpha^2+2L^2\alpha+2)\|\nabla f(w)\|^2 + \frac{\sigma^2}{2} \\
        & \leq \frac{5}{4} L^2\alpha(L^3\alpha^2+2L^2\alpha+2)\|\nabla f(w)\|^2 + \frac{\sigma^2}{2},
    \end{align*}
    where the last inequality follows from $\alpha < \frac{1}{2L}$.
\end{proof}

\paragraph{Proof of \cref{lem:upper-bound-dynamics}}

\begin{proof}
    Plug \cref{ineqn:spectral-bound2} into \cref{eqn:grad-f-dynamics} to get
    \begin{align*}
        \frac{d}{dt}\frac{1}{2}\bnorm{\nabla f(w)}^2 & \leq -\nabla f(w)\transpose\nabla^2 f(w)\nabla f(w) + \frac{5}{4} L^2\alpha(L^3\alpha^2+2L^2\alpha+2)\|\nabla f(w)\|^2 + \frac{\sigma^2}{2} \\
        & \leq -\Big( \mu - \frac{5}{4} L^2\alpha(L^3\alpha^2+2L^2\alpha+2) \Big) \|\nabla f(w)\|^2 + \frac{\sigma^2}{2}.
    \end{align*}
\end{proof}

\begin{theorem}\label{thm:spectral-bound}
    If it holds that
    \begin{align*}
        \alpha < \min\left\{ \sqrt[\leftroot{-2}\uproot{2}3]{\frac{2}{15}}\mu^{1/3}L^{-5/3}, \sqrt{\frac{1}{15}}\mu^{1/2}L^{-2}, \sqrt{\frac{1}{15}}\mu L^{-2} \right\},
    \end{align*}
    then $\|\nabla f(w(t))\|^2$ under \eqref{eqn:ode} is upper bounded by a function $y(t)$ that is exponentially convergent to
    \begin{align*}
        \frac{\sigma^2}{2\mu - \frac{5}{2} L^2\alpha(L^3\alpha^2+2L^2\alpha+2)} < \frac{\sigma^2}{\mu}
    \end{align*}
    as $t\to\infty$.
\end{theorem}

\begin{proof}
If $y(t)$ is the solution of an \ivp
\begin{align*}
    \dot y \leq -\Big( \mu - \frac{5}{4} L^2\alpha(L^3\alpha^2+2L^2\alpha+2) \Big) y + \frac{\sigma^2}{2}
\end{align*}
with initial condition $y(0) = \|\nabla f(w(0))\|^2$, then $\|\nabla f(w(t))\|^2 \leq y(t)$ for any $t\geq 0$. Moreover, it is an \ode of the following form: $\dot y = -\zeta y + \gamma$, which is a simple first-order separable \ode
    that permits a family of solutions
    \begin{align*}
        y(t) = (e^{-\zeta(t+c_0)} + \gamma)/\zeta
    \end{align*}
    under the condition $y(0) > \gamma/\zeta$. In our case, $\zeta = \mu - \frac{5}{4} L^2\alpha(L^3\alpha^2+2L^2\alpha+2)$, $\gamma = \frac{\sigma^2}{2}$, and the constant $c_0$ depends on initial condition $y(0)$. Consequently, we have $y$ converges to $\gamma/\zeta$ exponentially whenever $\zeta>0$.
    The following theorem provides sufficient conditions for convergence. 

    We derive sufficient conditions for the quadratic inequality $\frac{1}{2}\mu - \frac{5}{4} L^2\alpha(L^3\alpha^2+2L^2\alpha+2) > 0$, \ie,
    \begin{align*}
        \frac{5}{4}L^5\alpha^3 < \frac{\mu}{6}, \quad \frac{5}{2}L^4\alpha^2 < \frac{\mu}{6}, \quad \frac{5}{2}L^2\alpha < \frac{\mu}{6}.
    \end{align*}
    The sufficient conditions reduce to 
    \begin{align*}
        \alpha < \min\left\{ \sqrt[\leftroot{-2}\uproot{2}3]{\frac{2}{15}}\mu^{1/3}L^{-5/3}, \sqrt{\frac{1}{15}}\mu^{1/2}L^{-2}, \sqrt{\frac{1}{15}}\mu L^{-2} \right\}
    \end{align*}
    and we have
    \begin{align*}
        \frac{\gamma}{\zeta} < \frac{\sigma^2/2}{\mu/2} = \frac{\sigma^2}{\mu}.
    \end{align*}
\end{proof}

\begin{lemma}\label{lem:norm-bound}
    Suppose the loss function $f_i(w)$ satisfies \cref{assum:smoothness,assum:bounded-variance}, then for any $w\in\Rd$ such that $\|\nabla f(w)\| \leq G$, it holds that $\|\nabla F(w)\| \leq (1+2\alpha L + \alpha^2L^2)G + (2\alpha L + \alpha^2L^2)\sigma$.
\end{lemma}

\begin{proof}
    Recall that $\nabla F_i(w) = A_i(w)\nabla f_i(w - \alpha\nabla f_i(w))$. Apply mean value theorem to $\nabla f_i(w - \alpha\nabla f_i(w))$ to get
    \begin{align}\label{eqn:mean-value-thm}
        \begin{split}
            \nabla f_i(w - \alpha\nabla f_i(w)) & = \nabla f_i(w) - \alpha\nabla^2 f(\tilde w_i)\nabla f_i(w) \\
            & = A_i(\tilde w_i)\nabla f_i(w),
        \end{split}
    \end{align}
    where $\tilde w_i$ lies between $w$ and $w - \alpha\nabla f_i(w)$. Consequently, $\nabla F_i(w) = A_i(w)A_i(\tilde w_i)\nabla f_i(w)$.
    Further notice that
    \begin{align*}
        \|\nabla F(w)\| & = \|\expect_{i\sim p}\nabla F_i(w)\| \\ 
        & = \|\expect_{i\sim p}[\nabla f_i(w) + (\nabla F_i(w) - \nabla f_i(w))]\| \\
        & \leq \|\expect_{i\sim p}\nabla f_i(w)\| + \|\expect_{i\sim p}[(I-A_i(w)A_i(\tilde w_i))\nabla f_i(w)]\| \\
        & \leq \|\nabla f(w)\| + \expect_{i\sim p}[\|I_d - A_i(w)A_i(\tilde w_i)\|\|\nabla f_i(w)\|],
    \end{align*}
    The second equality follows from separating the difference between $\nabla F(w)$ and $\nabla f(w)$. The third inequality is due to \cref{eqn:mean-value-thm} and triangular inequality. The last inequality is due to Cauchy-Schwarz inequality, and the product of the two norms can be handled seperately. Expand $A_i(w)$, $A_i(\tilde w_i)$ and bound the first term by a constant to get
    \begin{align*}
        \|I_d - A_i(w)A_i(\tilde w_i)\| = \|\alpha^2\nabla^2 f_i(w)\nabla^2 f_i(\tilde w)-\alpha\nabla^2 f_i(w)-\alpha\nabla^2 f_i(\tilde w)\| \leq 2\alpha L + \alpha^2L^2.
    \end{align*}
    The remaining term can be bounded by variance $\sigma$ and gradient norm $\|\nabla f(w)\|$:
    \begin{align*}
        \expect_{i\sim p} \|\nabla f_i(w)\| & \leq \|\expect_{i\sim p}\nabla f_i(w)\| + \expect_{i\sim p}[\|\nabla f_i(w) - \expect_{i\sim p}\nabla f_i(w)\|] \\ 
        & \leq \|\nabla f(w)\| + \sqrt{\expect_{i\sim p}[\|\nabla f_i(w) - \nabla f(w)\|^2]} \\
        & \leq \|\nabla f(w)\| + \sigma.
    \end{align*}
    The second inequality follows from Jenson inequality.
    Combining the upper bounds together yields
    \begin{align*}
        \|\nabla F(w)\| \leq (1+2\alpha L + \alpha^2L^2)\|\nabla f(w)\| + (2\alpha L + \alpha^2L^2)\sigma.
    \end{align*}
\end{proof}

\paragraph{Proof of \cref{thm:main}}

\begin{proof}
    Recall that $y(t) = (e^{-\zeta(t+c_0)} + \gamma)/\zeta$, where we denote $y(0) = \|\nabla f(w(0))\|^2$, $\gamma = \sigma^2/2$, and $\zeta = \mu - \frac{5}{4} L^2\alpha(L^3\alpha^2+2L^2\alpha+2)$. Especially, under the assumptions of \cref{thm:spectral-bound}, we have $\frac{\mu}{2} < \frac{\mu}{2} + \iota \leq \zeta \leq \mu$ where $\iota > 0$ is a small constant.
    To achieve $\|\nabla f(w(t))\|^2 \leq \frac{\sigma^2}{\mu}$, it suffices to have
    \begin{align}
        y(t) = \frac{e^{-\zeta(t+c_0)} + \gamma}{\zeta} \leq \frac{\sigma^2}{\mu}. \label{ineqn:ode-convergence}
    \end{align}
    Since at initialization $e^{-\zeta c_0} = \zeta y(0) - \gamma$, as long as $y(0) = \|\nabla f(w(0))\|^2 > \frac{\sigma^2}{\mu}$, plug into \eqref{ineqn:ode-convergence} to get
    \begin{align*}
        e^{-\zeta t} & \leq \left( \frac{\zeta\sigma^2}{\mu} - \gamma \right) e^{\zeta c_0} \\
        t & \geq \frac{1}{\zeta}\log\left( \frac{\zeta y(0) - \frac{\sigma^2}{2}}{\frac{\zeta\sigma^2}{\mu} - \frac{\sigma^2}{2}} \right).
    \end{align*}
    Hence for any initialization $w(0)$ where $y(0) > \frac{\sigma^2}{\mu}$, we have $\|\nabla f(w(t))\|^2 \leq \frac{\sigma^2}{\mu}$ for any 
    \begin{align*}
        t \geq \frac{2}{\mu}\log\left( \frac{\mu y(0) - \frac{\sigma^2}{2}}{(\frac{\mu}{2}+\iota)\frac{\sigma^2}{\mu} - \frac{\sigma^2}{2}} \right) = \frac{2}{\mu}\log\left( \frac{\mu^2 \|\nabla f(w(0))\|^2 - \frac{\mu\sigma^2}{2}}{\iota\sigma^2} \right).
    \end{align*}
    For any initialization $w(0)$ where $y(0) \leq \frac{\sigma^2}{\mu}$, we skip the first phase and go directly into the second phase.
    
    Let us denote
    \begin{align*}
        t_1 = \min_t\left\{ t\st y(t) \leq \frac{\sigma^2}{\mu} \right\},
    \end{align*}
    and especially $t_1 = 0$ if $y(0) \leq \frac{\sigma^2}{\mu}$.
    By \cref{lem:norm-bound} and the assumption $\alpha \leq \frac{1}{2L}$ we have
    \begin{align*}
        \|\nabla F(w(t_1))\| & \leq (1+2\alpha L + \alpha^2L^2)\frac{\sigma}{\sqrt{\mu}} + (2\alpha L + \alpha^2L^2)\sigma \\
        & \leq \frac{1}{4}\left( \frac{9}{\sqrt{\mu}} + 5 \right)\sigma.
    \end{align*}
    Let us denote $K = \max\{ \frac{1}{4}\big( \frac{9}{\sqrt{\mu}} + 5 \big)\sigma, \big(\sqrt{\frac{L}{\mu }}+1\big)\sigma \}$, and \cref{thm:strongly-convex-F} implies that if $\alpha \leq \min\{\frac{1}{2L},\frac{\mu}{8\kappa(2K + \sigma)}\}$ the \maml loss $F(w)$ is $\frac{\mu}{8}$-strongly convex at $w$, and the \maml \ode \eqref{eqn:ode} after time $t_1$ is a gradient flow on a $\frac{\mu}{8}$-strongly convex loss $F(w)$. This dynamics then converges exponentially fast to an approximate stationary point $\widehat w$ where $\|\nabla F(\widehat w)\| \leq \varepsilon$. More specifically, we have 
    \begin{align*}
        \frac{d}{dt}\|\nabla F(w)\|^2 & = \nabla F(w)\transpose\nabla^2 F(w)\dot w \\
        & = -\nabla F(w)\transpose\nabla^2 F(w)\nabla F(w) \\
        & \leq -\frac{\mu}{8}\|\nabla F(w)\|^2.
    \end{align*}
    Note that even though the set of $w$ such that $\|\nabla F(w)\|\leq K$ is not necessarily convex, the trajectory of the \maml \ode still converges inside it. Consider a point $w(t_1)$ that has $\|\nabla F(w(t_1))\|\leq K$. If there exist $\tau>0$, such that $\|\nabla F(w(t_1+\tau))\| > K$, then we define $\tau_0 \defeq \max_\tau\{\tau\st \|\nabla F(w(t_1+t))\|\leq K, \forall 0\leq t \leq \tau\}$. Because $\nabla F(w)$ and $\nabla^2 F(w)$ are continuous in $w$, there exists a neighborhood around $w(t_1+\tau_0)$ such that for all $w$ in this neighborhood it holds $\|\nabla F(w)\| > K/2$ and $\|\nabla^2 F(w)\| > \mu/16$. Therefore it has $\frac{d}{dt}\|\nabla F(w)\|^2 < -\mu K^2/256$, and integrating it in a short time interval after $t_1+\tau_0$ yields a contradiction.
    Consequently, we obtain another upper bound
    \begin{align*}
        \|\nabla F(w(\tau+t_1))\|^2 \leq \tilde y(\tau+t_1) \defeq e^{-\mu \tau/8}\|\nabla F(w(t_1))\|^2
    \end{align*}
    for any time $\tau \geq 0$ after $t_1$. 
    A sufficient condition for the approximate stationary point $\widehat w$ writes $e^{-\mu \tau/8}\|\nabla F(w(t_1))\|^2 \leq \varepsilon^2$, which means $w(\tau+t_1)$ is an approximate stationary point if 
    \begin{align*}
        \tau & \geq \frac{8}{\mu}\log\left( \frac{\|\nabla F(w(t_1))\|^2}{\varepsilon^2} \right) \\
        & = \frac{16}{\mu}\log\left( \frac{(5+\frac{9}{\sqrt{\mu}})\sigma}{4\varepsilon} \right).
    \end{align*}
    
    Combine two parts together to get the major result that the \maml \ode converges to an approximate stationary point $\widehat w(t)$ within 
    \begin{align*}
        t = \frac{1}{\mu}\calO\left[ \log\left( \frac{(5+\frac{9}{\sqrt{\mu}})(\mu^2\sigma \|\nabla f(w(0))\|^2 - \frac{\mu\sigma^3}{2})}{4\iota\sigma^2\varepsilon} \right) \right].
    \end{align*}
    if $\|\nabla f(w(0))\|^2 > \frac{\sigma^2}{\mu}$, and the \maml \ode converges to an approximate stationary point $\widehat w(t)$ within
    \begin{align*}
        t = \frac{16}{\mu}\calO\left[ \log\left( \frac{(5+\frac{9}{\sqrt{\mu}})\sigma}{4\varepsilon} \right) \right]
    \end{align*}
    if $\|\nabla f(w(0))\|^2 \leq \frac{\sigma^2}{\mu}$.
\end{proof}

\paragraph{Proof of \cref{thm:bi-maml-convergence}}

\begin{proof}
    Since the expected loss $f$ is $\mu$-strongly convex, we always have in the first stage that
    \begin{align*}
        \frac{d}{dt}\|\nabla f(w)\|^2 & = \nabla f(w)\transpose\nabla^2 f(w)\dot w \\
        & = -\nabla f(w)\transpose\nabla^2 f(w)\nabla F(w) \\
        & \leq -\mu\|\nabla f(w)\|^2,
    \end{align*}
    where $\dot w = - \nabla f(w)$. It reaches a tolerant level at $\|\nabla f(w)\|\leq \varepsilon_0$, as long as
    \begin{align*}
        t & \geq \frac{1}{\mu}\log\left( \frac{\|\nabla f(w(0))\|^2}{\varepsilon_0^2} \right) \\
        & = \frac{2}{\mu}\log\left( \frac{\|\nabla f(w(0))\|^2}{\varepsilon_0^2} \right).
    \end{align*}
    
    Let us denote
    \begin{align*}
        t_1 = \min_t\left\{ t\st \|\nabla f(w(t))\|^2 \leq \varepsilon_0^2 \right\},
    \end{align*}
    By \cref{lem:norm-bound} and the assumption $\alpha \leq \frac{1}{2L}$ we have
    \begin{align*}
        \|\nabla F(w(t_1))\| & \leq (1+2\alpha L + \alpha^2L^2)\varepsilon_0 + (2\alpha L + \alpha^2L^2)\sigma \\
        & \leq \frac{9}{4}\varepsilon_0 + \frac{5}{4} \sigma.
    \end{align*}
    Let us denote $K = \max\{ \frac{9}{4}\varepsilon_0 + \frac{5}{4} \sigma, \big(\sqrt{\frac{L}{\mu }}+1\big)\sigma \}$, and \cref{thm:strongly-convex-F} implies that if $\alpha \leq \min\{\frac{1}{2L},\frac{\mu}{8\kappa(2K + \sigma)}\}$ the \maml loss $F(w)$ is $\frac{\mu}{8}$-strongly convex at $w$, and the \maml \ode \eqref{eqn:ode} after time $t_1$ is a gradient flow on a $\frac{\mu}{8}$-strongly convex loss $F(w)$. This dynamics then converges exponentially fast to an approximate stationary point $\widehat w$ where $\|\nabla F(\widehat w)\| \leq \varepsilon$. Similar to the proof of \cref{thm:main}, a sufficient condition for the approximate stationary point $\widehat w$ writes $e^{-\mu \tau/8}\|\nabla F(w(t_1))\|^2 \leq \varepsilon$, which means $w(\tau+t_1)$ is an approximate stationary point if 
    \begin{align*}
        \tau & \geq \frac{8}{\mu}\log\left( \frac{\|\nabla F(w(t_1))\|^2}{\varepsilon^2} \right) \\
        & = \frac{16}{\mu}\log\left( \frac{9\varepsilon_0 + 5 \sigma}{4\varepsilon} \right).
    \end{align*}
    
    Combine two parts together to get the major result that the \bmm \ode converges to an approximate stationary point $\widehat w(t)$ within 
    \begin{align*}
        t = \frac{1}{\mu}\calO\left[ \log\left( \frac{(9\varepsilon_0 + 5 \sigma)\|\nabla f(w(0))\|}{4\varepsilon_0\varepsilon} \right) \right].
    \end{align*}
\end{proof}

\section{Proof of Strong Convexity}\label{sec:strong-convexity}

\begin{lemma}\label{lem:bounded-f-gradient-norm}
    Suppose the loss function $f_i(w)$ satisfies \cref{assum:smoothness,assum:bounded-variance}, then for any $w\in\Rd$ such that $\|\nabla F(w)\| \leq K$ and $\alpha < \frac{1}{4L}$, it holds that $\|\nabla f(w)\| \leq 2K + \sigma$.
\end{lemma}

\begin{proof}
    Notice that
    \begin{align*}
        \|\nabla f(w)\| & = \|\expect_{i\sim p} f_i(w)\| \\
        & = \|\expect_{i\sim p}[\nabla F_i(w) + (\nabla f_i(w) - \nabla F_i(w))]\| \\
        & \leq \|\nabla F(w)\| + \|\expect_{i\sim p}(I_d - A_i(w)A_i(\tilde w_i))\nabla f_i(w)\| \\
        & \leq \|\nabla F(w)\| + \expect_{i\sim p}\|I_d - A_i(w)A_i(\tilde w_i)\|\|\nabla f_i(w)\| \\
        & \leq \|\nabla F(w)\| + 2\alpha L \expect_{i\sim p}\|\nabla f_i(w)\|,
    \end{align*}
    where the first inequality follows from triangular inequality and the last inequality is due to \cref{assum:smoothness} and $\alpha < \frac{1}{4L}$. Similarly, we have 
    \begin{align*}
        \expect_{i\sim p}\|\nabla f_i(w)\| & \leq \|\nabla f(w)\| + \expect_{i\sim p}\|\nabla f_i(w) - \nabla f(w)\| \\
        & \leq \|\nabla f(w)\| + \sigma,
    \end{align*}
    where the first inequality is due to triangular inequality and the second one is due to \cref{assum:bounded-variance}. Rearrange the terms under the assumption $\alpha < \frac{1}{4L}$ to get
    \begin{align*}
        \|\nabla f(w)\| & \leq \frac{1}{1 - 2\alpha L}\|\nabla F(w)\| + \frac{2\alpha L}{1 - 2\alpha L}\sigma \\
        & \leq 2K + \sigma.
    \end{align*}
\end{proof}

\begin{lemma}\label{lem:strong-convexity}
    Suppose $f_i(w)$ satisfies \cref{assum:smoothness,assum:lipschitz,assum:strongly-convex}. For any $\alpha \leq \min\{\frac{1}{2L},\frac{\mu}{8\kappa G}\}$ and  $w\in U(G)\defeq \{w\in \Rd: \|\nabla f(w)\|\leq G\}$, we have $\frac{\mu}{8}I_d\preceq \Hess(F(w)) \preceq \frac{9L}{8}I_d $.
\end{lemma}

\begin{proof}
    Consider $w,u\in U(G)$, we have
    \begin{align*}
        \|\nabla F(w) - \nabla F(u)\| & = \|A(w)\nabla f(w - \alpha\nabla f(w)) - A(u)\nabla f(u - \alpha\nabla f(u))\| \\
        & \leq \|(A(w)-A(u))\nabla f(w - \alpha\nabla f(w))\| \\
        & \qquad + \|A(u)(\nabla f(w - \alpha\nabla f(w))-\nabla f(u - \alpha\nabla f(u)))\|,
    \end{align*}
    where the inequality follows from triangular inequality. For the first term, we have an upper bound
    \begin{align*}
        \|(A(w)-A(u))\nabla f(w - \alpha\nabla f(w))\| & \leq \|A(w)-A(u)\|\|\nabla f(w - \alpha\nabla f(w))\| \\
        & = \alpha\|\nabla^2 f(w) - \nabla^2 f(u)\|\|\nabla f(w - \alpha\nabla f(w))\| \\
        & \leq \alpha\kappa\|w-u\|\|\nabla f(w - \alpha\nabla f(w))\| \\
        & = \alpha\kappa\|w-u\|\|A(\tilde w)\nabla f(w)\| \\
        & \leq \alpha\kappa\|w-u\|\|A(\tilde w)\|\|\nabla f(w)\| \\
        & \leq \alpha(1-\alpha\mu)\kappa G \|w-u\|
    \end{align*}
    where the first inequality is due to Cauchy-Schwarz inequality, the second inequality is due to \cref{assum:lipschitz}, and the second equality follows from mean value theorem, and the last inequality is due to the fact that $\|A(\tilde w)\| = \|I_d - \alpha\nabla^2 f(\tilde w)\| \leq 1 - \alpha\mu$. Similarly, we bound the second part as
    \begin{align*}
        & \quad\ \|A(u)(\nabla f(w - \alpha\nabla f(w))-\nabla f(u - \alpha\nabla f(u)))\| \\ 
        & \leq \|A(u)\|\|\nabla f(w - \alpha\nabla f(w))-\nabla f(u - \alpha\nabla f(u))\| \\
        & \leq (1-\alpha\mu)\|\nabla f(w - \alpha\nabla f(w))-\nabla f(u - \alpha\nabla f(u))\| \\
        & \leq (1-\alpha\mu)L\|(w-\alpha\nabla f(w))-(u-\alpha\nabla f(u))\| \\
        & \leq (1-\alpha\mu)^2L\|w-u\|,
    \end{align*}
    where the last inequality follows from mean value inequality. Putting the pieces together to get, when $\alpha \leq \min\{\frac{1}{2L},\frac{\mu}{8\kappa G}\}$,
    \begin{align*}
        \|\nabla F(w) - \nabla F(u)\| & \leq \alpha(1-\alpha\mu)\kappa G \|w-u\| + (1-\alpha\mu)^2L\|w-u\| \\
        & \leq \alpha\kappa G \|w-u\| + (1-\alpha\mu)^2L\|w-u\| \\
        & \leq \left( \frac{\mu}{8} + L \right)\|w-u\| \\
        & \leq \frac{9L}{8}\|w-u\|,
    \end{align*}
    and therefore $\Hess(F(w)) \preceq \frac{9L}{8} I_d$.
    
    The corresponding lower bound similarly follows from triangular inequality where
    \begin{align*}
        \|\nabla F(w) - \nabla F(u)\| & = \|A(w)\nabla f(w - \alpha\nabla f(w)) - A(u)\nabla f(u - \alpha\nabla f(u))\| \\
        & \geq \|A(u)(\nabla f(w - \alpha\nabla f(w))-\nabla f(u - \alpha\nabla f(u)))\| \\
        & \qquad - \|(A(w)-A(u))\nabla f(w - \alpha\nabla f(w))\|.
    \end{align*}
    When $\alpha \leq \min\{\frac{1}{2L},\frac{\mu}{8\kappa G}\}$, the first term is lower bounded as
    \begin{align*}
        & \quad\ \|A(u)(\nabla f(w - \alpha\nabla f(w))-\nabla f(u - \alpha\nabla f(u)))\| \\ 
        & \geq (1-\alpha L)\|\nabla f(w - \alpha\nabla f(w))-\nabla f(u - \alpha\nabla f(u))\| \\
        & \geq (1-\alpha L)\mu\|(w - \alpha\nabla f(w)) - (u - \alpha\nabla f(u))\| \\
        & \geq (1-\alpha L)\mu(\|w-u\|-\alpha\|\nabla f(w)-\nabla f(u)\|) \\
        & \geq (1-\alpha L)^2\mu\|w-u\| \\
        & \geq \frac{\mu}{4}\|w-u\|,
    \end{align*}
    where the first inequality follows from $\lambdamin(A(u))\geq 1-\alpha L$, the second inequality follows from \cref{assum:strongly-convex}, the third inequality is due to triangular inequality, and the last inequality follows from $\alpha \leq \frac{1}{2L}$. Hence, it holds that
    \begin{align*}
        \|\nabla F(w) - \nabla F(u)\| & \geq \|A(u)(\nabla f(w - \alpha\nabla f(w))-\nabla f(u - \alpha\nabla f(u)))\| \\
        & \qquad - \|(A(w)-A(u))\nabla f(w - \alpha\nabla f(w))\| \\
        & \geq \frac{\mu}{4}\|w-u\| - \alpha(1-\alpha\mu)\kappa G \|w-u\| \\
        & \geq \left( \frac{\mu}{4}-\frac{\mu}{8} \right)\|w-u\| \\
        & = \frac{\mu}{8}\|w-u\|,
    \end{align*}
    where the last inequality follows from $\alpha \leq \frac{\mu}{8\kappa G}$. Thus we obtain $\Hess(F(w))\ge \frac{\mu}{8}$. 
\end{proof}

\paragraph{Proof of \cref{thm:strongly-convex-F}}

\begin{proof}
    Combining \cref{lem:bounded-f-gradient-norm,lem:strong-convexity} shows that
    \[ \frac{\mu}{8}I_d \preceq \Hess(F(w)) \preceq \frac{9L}{8}I_d\,,\]
    if $w\in U(K)$ and
    \begin{align*}
        \alpha \leq \min\left\{ \frac{1}{2L},\frac{\mu}{8\kappa(2K + \sigma)} \right\}.
    \end{align*}
\end{proof}

\section{Proof of \cref{thm:F-unique-critical-point}}\label{sec:proof-F-unique-critical-point}

For $K>0$, we define $U(K)\defeq \{ w\in \RR^d: \|\nabla F(w)\|\le K \} $ and $V(K)\defeq \{  w\in \RR^d: f(w) - f(x^*)\le  K \} $ where $x^*$ is the unique global minimizer of $f$ (recall that $f$ is $\mu$-strongly convex). 
Let $\crit(F)$ denote the set of critical points of $F$. 
The convexity of $f$ implies that $V(K)$ is convex. All critical points of $F$ are contained in $U(K)$ for any $K>0$; in other words
\[
\crit(F) \subseteq U(K),\quad\forall K>0\,.
\]

\begin{lemma}\label{lem:U-contained-in-V}
    If the loss function $f_i(w)$ satisfies \cref{assum:smoothness,assum:bounded-variance,assum:strongly-convex}, $\alpha<\frac{1}{4L}$, then we have \[
    U(K) \subseteq V\left(\frac{1}{2\mu}\left(2K + \sigma \right)^2\right)\,.
    \]
\end{lemma}
\begin{proof}
    Let us pick $w\in \RR^d$ such that $\|\nabla F(w)\|\le K$. 
    \cref{lem:bounded-f-gradient-norm} implies that there exists a constant $C_1=2K + \sigma$ such that $\|\nabla f(w)\| \le C_1$. Since $f$ is $\mu$-strongly convex, we have 
    \[
    f(w) \le f(x) + \nabla f(x)\transpose (w-x) + \frac{1}{2\mu}\| \nabla f(w)-\nabla f(x) \|^2,\quad \forall w,x\,.
    \]
    Setting $x$ to the global minimizer $x^*$ of $f$ yields \[
    f(w) \le f(x^*) + \frac{1}{2\mu}\| \nabla f(w) \|^2 \le f(x^*) + \frac{1}{2\mu}C_1^2 = f(x^*) + \frac{1}{2\mu}\left(2K + \sigma \right)^2 \,.
    \]
    Therefore, we have \[
    w\in V\left(\frac{1}{2\mu}\left(2K + \sigma \right)^2\right)\,.
    \]
\end{proof}

\begin{lemma}\label{lem:V-contained-in-U}
    Under \cref{assum:smoothness,assum:bounded-variance,assum:strongly-convex}, we have \[
        V(K) \subseteq U \left(\sigma + \sqrt{2LK}\right)\,.
    \]
\end{lemma}
\begin{proof}
    Let us rewrite $\|\nabla F(w)\|$ as below \begin{equation}\label{eq:gradF-bounded-by-gradf}
    \begin{split}
        \|\nabla F(w)\| ={}& \|\expect_{i\sim p}(I_d-\alpha \nabla^2 f_i(w) )\nabla f_i(w-\alpha f_i(w)) \|\\
    ={}& \|\expect_{i\sim p}(I_d-\alpha \nabla^2 f_i(w) )(I_d-\alpha \nabla^2 f_i(\tilde{w}) )\nabla f_i(w) \|\\
    \le{}& \expect_{i\sim p} \|\nabla f_i(w) \|\\
    \le{}&  \left( \expect_{i\sim p} \|\nabla f_i(w) - f(w) \| + \|\nabla f(w)\| \right)\\
    \le{}&   \sqrt{\expect_{i\sim p} \|\nabla f_i(w) - f(w) \|^2} + \| \nabla f(w)\| \\
    \le{}&   \sigma + \|\nabla f(w)\|\,,
    \end{split}
    \end{equation}
    where the second inequality is because of the mean value theorem. Since $f$ is $L$-smooth, we have  
    \[
    f(w) \ge f(x) + \nabla f(x)\transpose (w-x) + \frac{1}{2L} \| \nabla f(w) - \nabla f(x) \|^2,\quad \forall x\in \RR^d\,.
    \]
    Since $f$ is $\mu$-strongly convex, there exists a unique global minimum $x^*$ with $\nabla f(x^*) = 0$. Therefore, we obtain \[
    f(w) \ge f(x^*) + \frac{1}{2L} \|\nabla f(w)\|^2\,.
    \]
    Combining the above inequality and \eqref{eq:gradF-bounded-by-gradf} yields \[
    \|\nabla F(w) \| \le \sigma + \sqrt{2L(f(w)-f(x^*))}\,.
    \]
    If $w\in V(K)$, we get \[
    \|F(w)\| \le  \sigma + \sqrt{2LK}\,.
    \]
\end{proof}
Combining \cref{lem:U-contained-in-V,lem:V-contained-in-U} gives the following corollary.
\begin{corollary}\label{cor:inclusion}
    For any $K>0$, if $\alpha<\frac{1}{4L}$, we have the following inclusion relations \[
    \crit(F)\subseteq U(K) \subseteq V\left(\frac{1}{2\mu}\left(2K + \sigma \right)^2\right) \subseteq U\left(\sigma + \sqrt{\frac{L}{\mu}}(2K + \sigma)\right).
    \]
\end{corollary}
\begin{corollary}\label{cor:UK-inclusion}
    For any $K'> \left(\sqrt{\frac{L}{\mu }}+1\right)\sigma$, if $\alpha <\frac{1}{4L}$, we have the following inclusion relations \[
    \crit(F)\subseteq U\left(\frac{K'-\sigma  \left(\sqrt{\frac{L}{\mu }}+1\right)}{2 \sqrt{\frac{L}{\mu }}}\right) \subseteq V\left(\frac{(K'-\sigma )^2}{2 L}\right) \subseteq U(K')
    \]
\end{corollary}
\begin{lemma}\label{lem:critical-point-existence}
    Under \cref{assum:strongly-convex}, if $\alpha<\frac{1}{4L}$, we have $\crit(F)$ is non-empty.
\end{lemma}
\begin{proof}
    First we show that $F$ is bounded from below. Since every $f_i$ is strongly convex, it is bounded from below. Recall that $F(w) = \expect_{i\sim p} f_i(w-\alpha \nabla f_i(w))$. Therefore $F$ is also bounded from below. Let $F^* \defeq \inf_{w\in \RR^d} F(w)$. Pick any $v(0)\in \RR^d$ and consider the dynamic defined by \[
        \frac{d v(t)}{d t} = - \nabla F(v(t))\,.
    \]
    Let $E(t) = F(v(t)) - F^*$. We have \[
        \frac{dE(t)}{dt} = -\| \nabla F(v(t)) \|^2\,.
    \]
    Therefore, we get \[
    t\min_{0\le s\le t} \| \nabla F(v(t)) \|^2 \le \int_0^t \| \nabla F(v(s)) \|^2 ds = E(0) - E(t) \le E(0)\,.
    \]
    Thus we obtain \begin{equation}\label{eq:min-norm-shrink}
        \min_{0\le s\le t} \| \nabla F(v(t)) \|^2 \le \frac{E(0)}{t}\,.
    \end{equation}
    Define another function \[
        u(t) \defeq v\left(\argmin_{s \in [0,t]} \| \nabla F(v(t)) \|^2\right)\,,
    \]
    where ties can be broken arbitrarily. \cref{eq:min-norm-shrink} implies \[
        \| \nabla F(u(t)) \| \le \sqrt{\frac{E(0)}{t}},\quad \forall t\ge 0\,.
    \]
    Pick any $K\ge \left(\sqrt{\frac{L}{\mu }}+1\right)\sigma$. We have \[
        \| \nabla F(u(t)) \| \in U(K),\quad \forall t\ge \sqrt{\frac{E(0)}{K}}\,.
    \]
    Since $f$ is strongly convex, $V\left(\frac{(K-\sigma )^2}{2 L}\right)$ is convex and non-empty. Thus $U(K)$ is non-empty and closed. 
    Next, we show that $U(K)$ is bounded. \cref{lem:U-contained-in-V} implies $U(K) \subseteq V\left(\frac{1}{2\mu}\left(2K + \sigma \right)^2\right)\defeq V_0$. Since $V_0$ is a sublevel set of $f$ and $f$ is strongly convex, therefore we get the boundedness of $V_0$, which implies the boundedness of $U(K)$. Thus $U(K)$ is compact.  
    Define a sequence $w_n = u\left(n+\sqrt{\frac{E(0)}{K}}\right)$, where $n=1,2,3,\dots$. We have $w_n\in U(K)$. By Bolzano-Weierstrass theorem, there exists a convergent subsequence $w_{n_i}$. Let $w_0\in U(K)$ be the limit of $w_{n_i}$. We have \[
    \|\nabla F(w_0)\| = \lim_{i\to \infty} \|\nabla F(w_{n_i})\| \le \lim_{i\to \infty} \sqrt{\frac{E(0)}{n_i+\sqrt{E(0)/K}}} = 0\,.
    \]
    Therefore we conclude that $w_0$ is a critical point of $F$.
\end{proof}

\begin{proof}[Proof of \cref{thm:F-unique-critical-point}]
    Since $f$ is strongly convex, $V\left(\frac{(K-\sigma )^2}{2 L}\right)$ is convex and non-empty. \cref{thm:strongly-convex-F} implies that 
    $  \Hess(F(w)) \succeq \frac{\mu}{8}I_d$ holds on $U(K)$ 
    and therefore $\frac{\mu}{8}$-strongly convex on its convex subset $V\left(\frac{(K-\sigma )^2}{2 L}\right)$ (by \cref{cor:UK-inclusion}). \cref{cor:UK-inclusion} also shows that $V\left(\frac{(K-\sigma )^2}{2 L}\right)$ contains all critical points of $F$ if $K\ge \left(   \sqrt{\frac{L}{\mu }}+1\right)\sigma$.
    Since $\crit(F)\ne \varnothing$ (by \cref{lem:critical-point-existence}), there is a unique critical point which is the minimizer of $F$ on $V\left(\frac{(K-\sigma )^2}{2 L}\right)$. \cref{cor:UK-inclusion} implies no critical point outside $V\left(\frac{(K-\sigma )^2}{2 L}\right)$. In fact, the unique critical point is the global minimizer of $F$. 
\end{proof}

\end{document}

%% file: introduction.tex
\section{Introduction}\label{sec:intro}

In machine learning, an ideal learner is able to speed up the learning of new tasks based upon previous experiences. This goal has been shared among different but highly related approaches such as  few-shot learning \citep{snell2017prototypical}, domain adaptation~\citep{yu2018one,li2018learning}, transfer learning \citep{pan2009survey}, and meta  learning (\aka \emph{learning to learn}) ~\citep{thrun2012learning,naik1992meta,schmidhuber1987evolutionary,hochreiter2001learning}. In particular,  meta-learning addresses a general optimization framework that aims to learn a model based on data from previous tasks, so that the learned model, after fine-tuning, can adapt to and perform well on new tasks. This idea has been successfully applied to different learning scenarios including reinforcement learning ~\citep{wang2016learning,duan2016rl,schweighofer2003meta}, deep probabilistic models~\citep{edwards2016towards,lacoste2017deep,grant2018recasting}, language models~\citep{radford2019language}, imitation learning~\citep{duan2017one}, and unsupervised learning~\citep{hsu2018unsupervised}, to name a few.  Previous works have also investigated meta learning  from a variety of perspectives and methods, including memory-based neural networks~\citep{santoro2016meta,munkhdalai2017meta}, black-box optimization ~\citep{duan2016rl,wang2016learning1}, learning to design an optimization algorithm~\citep{andrychowicz2016learning},  attention-based models~\citep{vinyals2016matching}, and LSTM-based learners~\citep{ravi2016optimization}. 

One of the gradient-based meta-learning algorithms that has been widely used and enjoys great empirical successes is the \emph{model-agnostic meta-learning} proposed in \citep{Finn2017}. 
Rather than minimizing directly on a combination of task losses, \maml meta-trains by minimizing the loss evaluated at one or multiple gradient descent steps further ahead for each task. The intuition is as follows: First, this meta-trained initialization resides close to the best parameters of all training tasks. Second, for each new task, the model can then be easily fine-tuned for that specific new task with only a small number of gradient descent steps. 
Following this work, many experimental and theoretical studies of \maml have been carried out~\citep{song2019maml,behl2019alpha,grant2018recasting,yang2019norml,rajeswaran2019meta}. In particular, \citet{fallah2019convergence} and \citet{mendonca2019guided} investigated the convergence of \maml to first order stationary points. Moreover, \citet{finn2019online} studied an online variant of \maml and \citet{antoniou2018train} proposed various modifications to \maml to improve its performance.  Empirically, \citet{raghu2019rapid} found that the success of \maml can be primarily associated to feature reuse, given high quality representations provided by the meta-initialization.

\maml assumes a shared parameter space $\Rd$ among all the tasks and learns an initialization $\widehat w\in\Rd$ from a batch of tasks at \emph{meta-learning} time, such that performing a few steps of gradient descent from this initialization at \emph{meta-testing} time minimizes the loss function for a new task on a smaller dataset. More specifically, in a task pool with total $M$ candidate tasks,  each task $\calT_i$ is sampled from the pool according to a distribution $p(\calT)$ and has a corresponding risk function $f_i:\Rd\to\RR$ that is parameterized by a shared variable $w\in\Rd$. This function measures the performance of the parameter $w$ on task $\calT_i$. In practice we estimate this risk function from a set of training data for each task. However, for simplicity, we directly work with the risk function in this paper and assume that we can acquire the knowledge of the task loss $f_i$ from an oracle. At meta-learning time, \maml looks for a warm start $\widehat w$ via solving an optimization problem that minimizes the expected loss over $f_i$ after one step of gradient descent, \ie,
\begin{align}\label{eqn:maml-opt}
    \widehat w = \argmin_w \expect_{i\sim p}[f_i(w - \alpha\nabla f_i(w))],
\end{align}
where $\expect_{i\sim p}$ denotes the expectation over sampled tasks and $\alpha$ represents the \maml step size (also referred to as the \maml parameter).
To ease the burden on notation, we define the \maml loss function of task $\calT_i$ to be $F_i(w) \defeq f_i(w - \alpha\nabla f_i(w))$. Further, we define some shorthand for the expected loss $f(w) \defeq \expect_{i\sim p}f_i(w)$ and the expected \maml loss $F(w) \defeq \expect_{i\sim p}F_i(w)$. Now we are able to represent the optimization problem \eqref{eqn:maml-opt} in a more concise form: $\argmin_w F(w)$.

Solving \eqref{eqn:maml-opt} yields a solution $\widehat w$ that in expectation serves as a good initialization point, such that after a step of gradient descent it achieves the best average performance over all possible tasks. However, because the objective function $F(w)$ is non-convex in general, there is no guarantee that one is able to find a global minimum. Therefore it is common to instead consider finding an approximate stationary point $\widehat w$ where $\|\nabla F(\widehat w)\| \leq \varepsilon$ for some small $\varepsilon$ \citep{fallah2019convergence}.

In this paper, we study the \maml algorithm proposed in \citep{Finn2017}, in which \maml meta-learns the model by performing gradient descent on the \maml loss. We begin by establishing a smooth approximation of this discrete-time iterative procedure and considering the continuous-time limit by taking an infinitesimally small step size. With any initialization of the parameter, the problem is transformed into an \emph{initial value problem} (\ivp) of an \ode, which is the underlying dynamic that governs the training of a \maml model. Specifically, We consider the convergence of the gradient norm $\|\nabla F(w)\|$ and the function value $F(w)$ under the continuous-time limit of \maml.

Our contributions can be summarized as follows:
\begin{itemize}
    \item
    We propose an \ode that underlies the training dynamics of a \maml model. In this manner, we eliminate the influence of the manually chosen step size. %
    The algorithm in \citep{Finn2017} can be viewed as an forward Euler integration of this \ode.
    \item 
    We prove that the aforementioned \ode enjoys a linear convergence rate to an approximate stationary point of $F$ for strongly convex task losses. In particular, it achieves an improved convergence rate of $\calO(\log\frac{1}{\varepsilon})$, compared to $\calO(\frac{1}{\varepsilon^2})$ in \citep{fallah2019convergence}. Note that the \maml loss $F$ may not be convex even if the individual task losses $f_i$ are strongly convex.
    \item
    We show that for strongly convex task losses with moderate regularization conditions, the \maml loss $F$ has a unique critical point that is also the global minimum.
    \item 
    We propose a new algorithm named \bmm for strongly convex task losses that is   computationally efficient and enjoys a similar linear convergence rate to the global minimum compared to the original \maml. It also converges significantly faster than \maml on a variety of learning tasks.
\end{itemize}

%% file: experiments.tex
\section{Numerical Experiments}\label{sec:experiments}

Our experiments evaluate the \bmm algorithm proposed in \cref{sec:main-results} against the \maml algorithm on a series of learning problems. More specifically, we compare both methods on two different tasks: linear regression and binary classification with a Support Vector Machine (\svm). They correspond to two different types of task loss $f_i$: strongly convex and convex. For both types of problem we compare the methods on both synthetic and real data.

\paragraph{Linear regression.} For the linear regression problem with synthetic data, we meta-learn the model parameter $\gamma\in\Rd$ from a set of generated data that correspond to $M = 10$ individual linear regression tasks, each with dimension $d = 20$. A ground truth vector $\gamma_i\in\Rd$ is generated independently for each individual task. Each $\gamma_i$ has its coordinates drawn from i.i.d.\ standard normal distributions. From each task, we generate $n = 100$ training samples $\{x_j,y_j\}\in\Rd\times\RR$, where each entry of $x_j$ is subject to an i.i.d.\ standard normal distribution and $y_j = x_j\transpose\gamma_i + \sigma Z_j$ where $Z_j$ is a standard normal variable and $\sigma = 1$ in our case. We also initialize the model parameter $\gamma\in\Rd$ randomly in the way we generate $\gamma_i$, and the \maml step size $\alpha = 0.3$ is fixed for all tasks. Each linear regression task $i$ has a strongly convex loss function $f_i(\gamma) = \|y - X\gamma\|^2$, where $X = [x_1,\ldots,x_n]\transpose$ and $y = [y_1,\ldots,y_n]\transpose$. We take the gradient descent step size $\beta = 0.05$ for both \bmm and \maml, and we present the numerical results in \cref{fig:lr_function,fig:lr_gradient}. 
For the linear regression problem on the \texttt{Diabetes} data set, we divide the original data into $M = 4$ tasks with dimension $d = 8$ according to male/female and whether the person is older than mean age of the data. We initialize the model parameter $\beta$ randomly with i.i.d.\ normal entries, and the \maml step size $\alpha = 0.5$ is fixed for all tasks. We choose $\beta = 0.1$ and run the algorithms. Numerical results are presented in \cref{fig:lr-diabetes-function,fig:lr-diabetes-gradient}.
Compared to \maml, our \bmm algorithm in \cref{fig:lr_function} converges to a neighborhood of stationary point of $F$ in 25 steps, where \maml is far from convergence. After 50 iterations, \bmm also shows its superiority in computational efficiency over \maml by having 37\% and 36\% decrease in computation time on synthetic and real data, respectively.

\paragraph{Support vector machine.} For binary classification with an \svm, we meta-learn from $M = 50$ binary classification tasks in $d = 20$ dimensional space, where each one of the individual tasks is composed of $n = 300$ samples $\{x_j,y_j\}\in\Rd\times\{\pm 1\}$, evenly split between the positive and negative classes. We generate these training data sets with the default data generating function from \texttt{scikit-learn}. For each task $i$, it is equipped with a convex (but not strongly convex) hinge loss $f_i(w) = \ReLu(y - X w)$, where $X = [x_1,\ldots,x_n]\transpose$ and $y = [y_1,\ldots,y_n]\transpose$. During the experiment, we take the \maml parameter $\alpha = 0.3$ for all tasks, and the gradient descent step size $\beta = 0.1$ is used for both the \bmm and the \maml. The corresponding numerical results are presented in \cref{fig:svm_function,fig:svm_gradient}. 
For the binary classification on \texttt{MNIST}, we design $M = 5$ tasks, each to be a classification on different digits (classify $1,3,5,7,9$ against $2,4,6,8,0$), which has dimension $d = 784$. Each task has $n = 200$ samples, and the \maml parameter $\alpha = 0.3$ is fixed for all tasks. We take $\beta = 0.01$ for the experiment, and the result is reported in \cref{fig:svm-mnist-function,fig:svm-mnist-gradient}.
Compared to \maml, our \bmm algorithm in \cref{fig:svm_gradient} reaches the region with small gradient norm $\|\nabla F(w)\|$ in less than 30 steps, where \maml still has a much larger gradient. In these experiments, \bmm shows its superiority in computational efficiency over \maml by having a 19\% and 37\% decrease in computation time on the synthetic and real data, respectively.

\begin{figure}[ht]
\centering

\begin{subfigure}[b]{.242\textwidth}
  \centering
  \includegraphics[width=\linewidth]{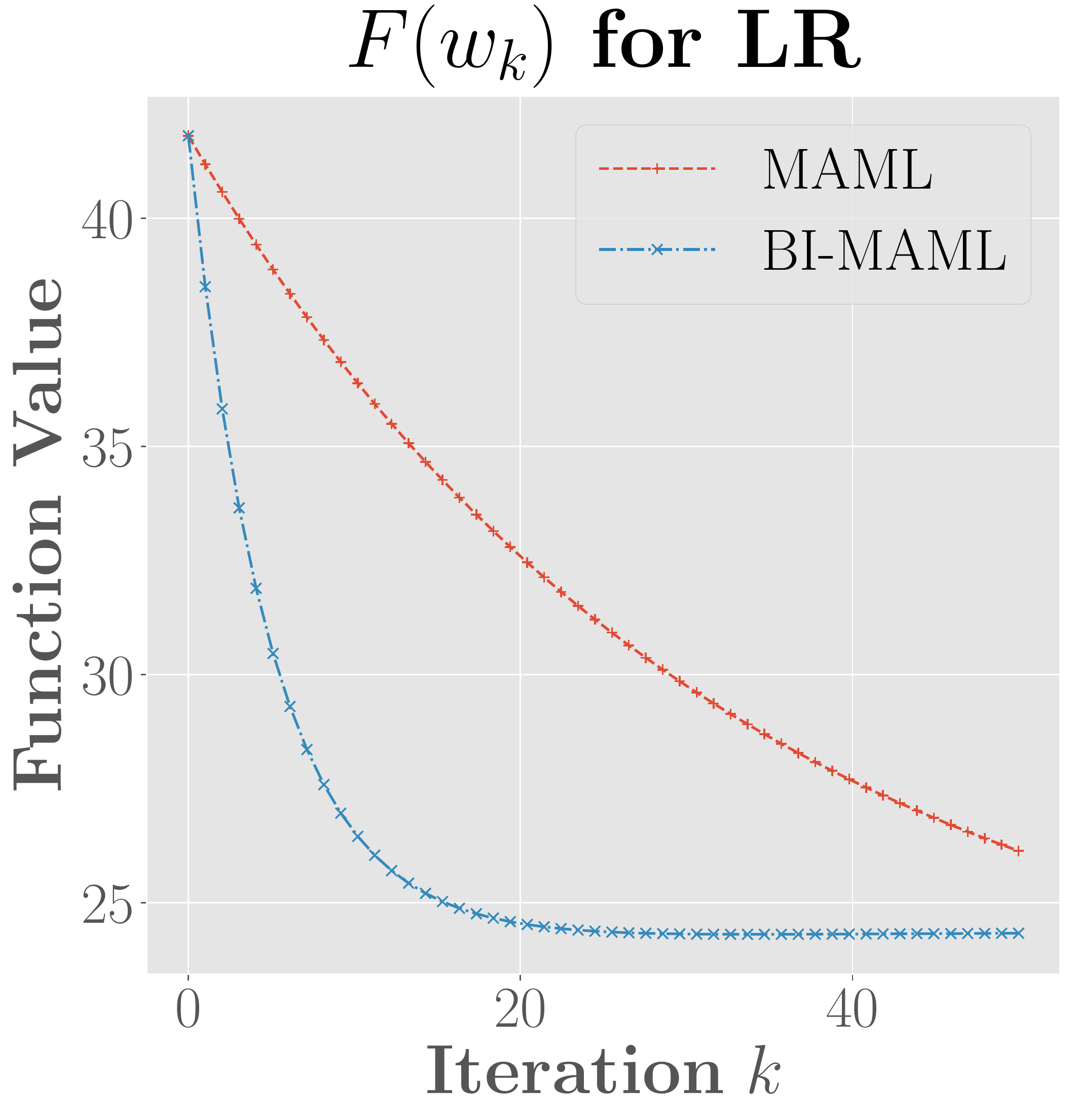}
  \caption{Linear regression on synthetic data}
  \label{fig:lr_function}
\end{subfigure}\hfill
\begin{subfigure}[b]{.242\textwidth}
  \centering
  \includegraphics[width=\linewidth]{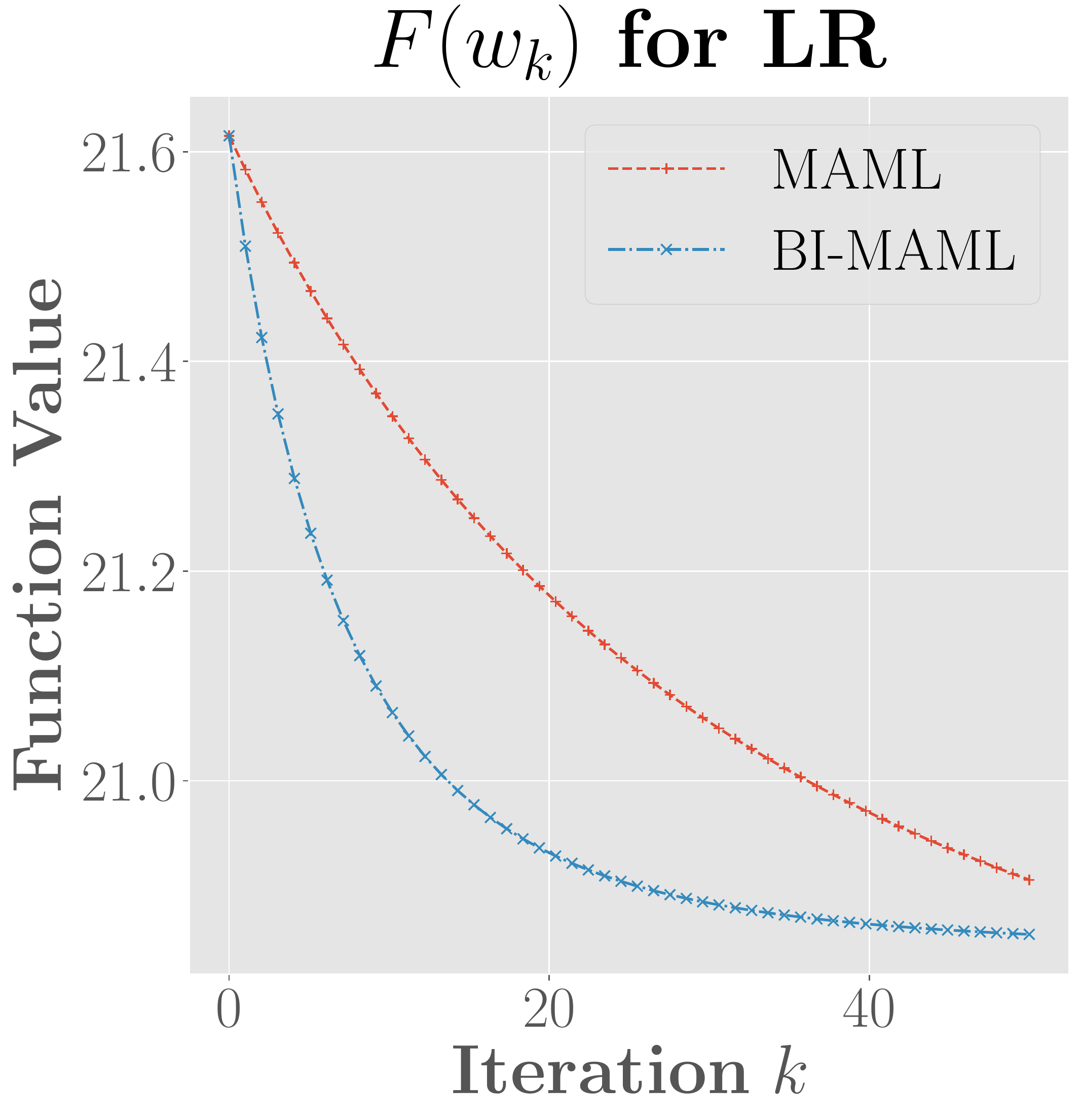}
  \caption{Linear regression on \texttt{Diabetes}}
  \label{fig:lr-diabetes-function}
\end{subfigure}\hfill
\begin{subfigure}[b]{.242\textwidth}
  \centering
  \includegraphics[width=\linewidth]{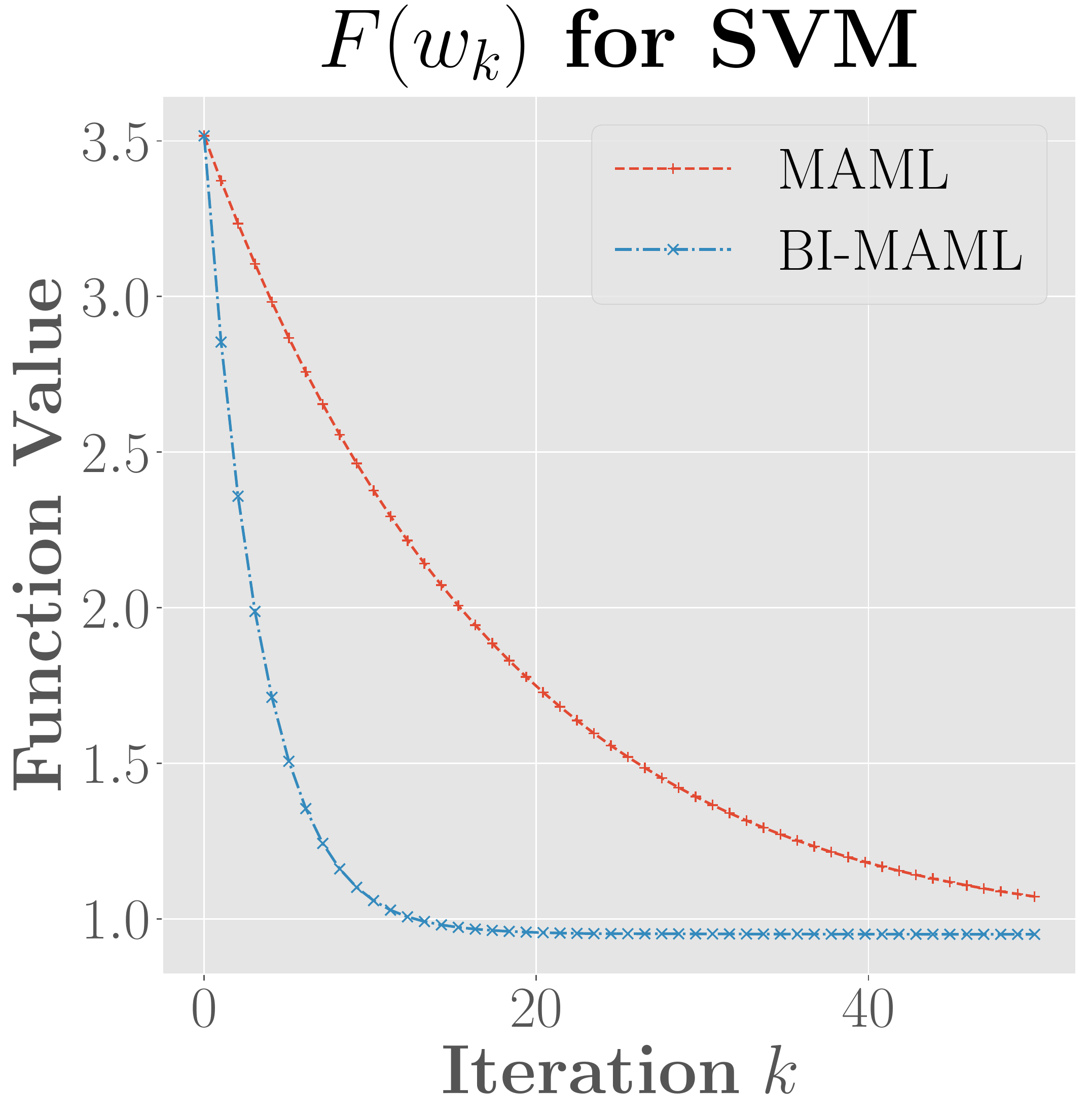}
  \caption{Binary class.~with \svm on synthetic data}
  \label{fig:svm_function}
\end{subfigure}\hfill
\begin{subfigure}[b]{.242\textwidth}
  \centering
  \includegraphics[width=\linewidth]{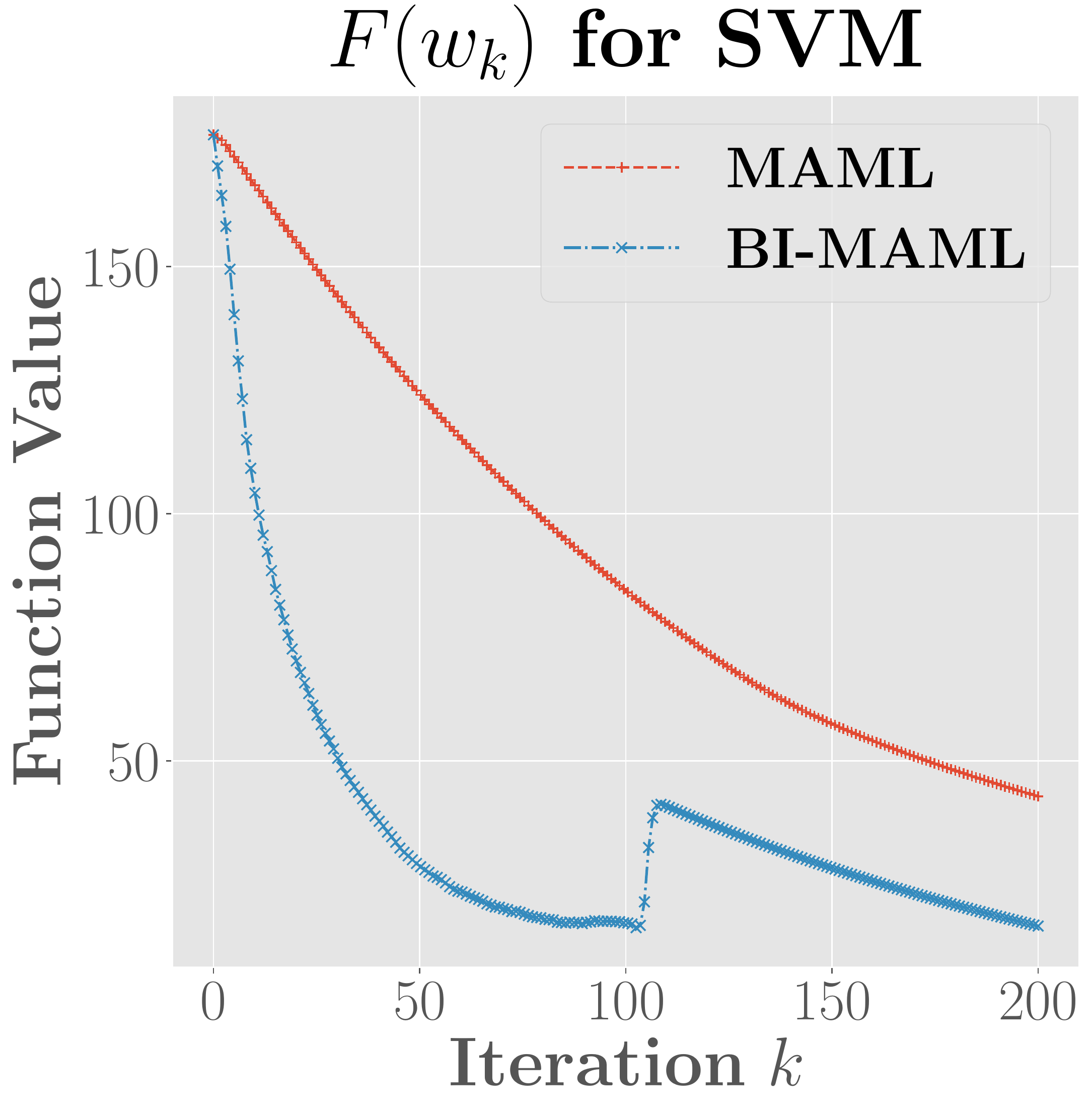}
  \caption{Binary classification with \svm on \texttt{MNIST}}
  \label{fig:svm-mnist-function}
\end{subfigure}

\medskip

\begin{subfigure}[b]{.242\textwidth}
  \centering
  \includegraphics[width=\linewidth]{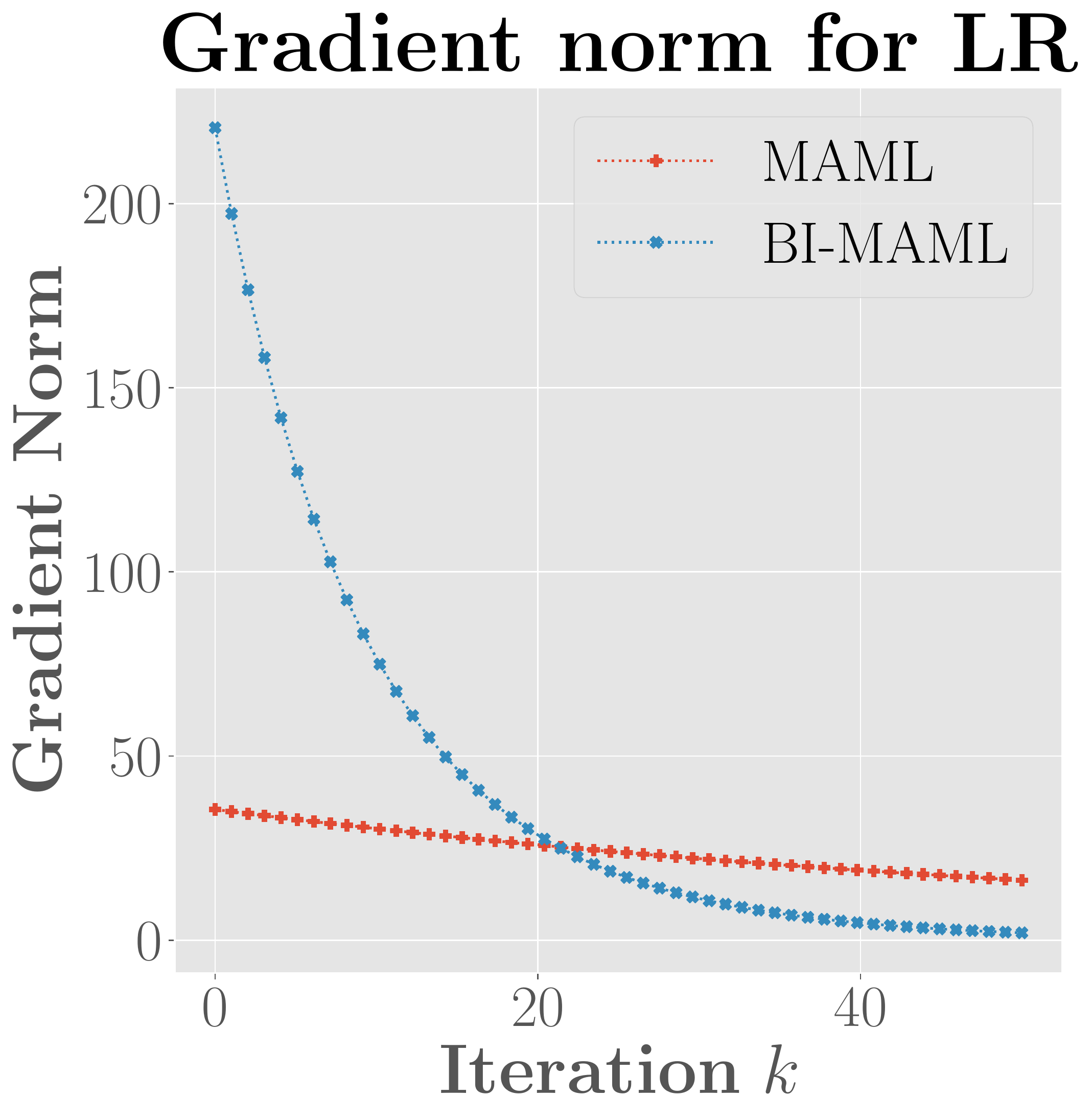}
  \caption{Linear regression on synthetic data}
  \label{fig:lr_gradient}
\end{subfigure}\hfill
\begin{subfigure}[b]{.242\textwidth}
  \centering
  \includegraphics[width=\linewidth]{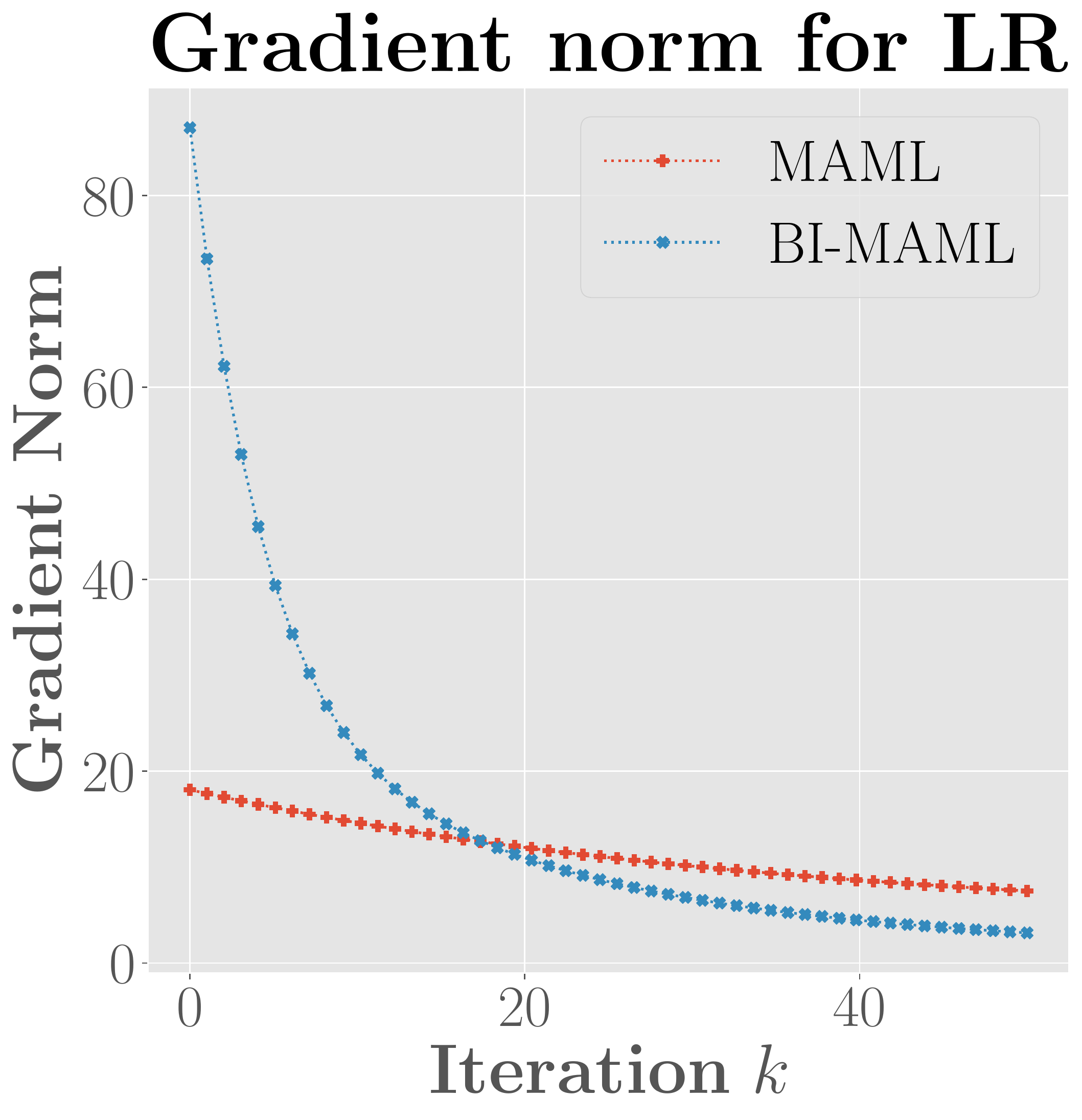}
  \caption{Linear regression on \texttt{Diabetes}}
  \label{fig:lr-diabetes-gradient}
\end{subfigure}
\begin{subfigure}[b]{.242\textwidth}
  \centering
  \includegraphics[width=\linewidth]{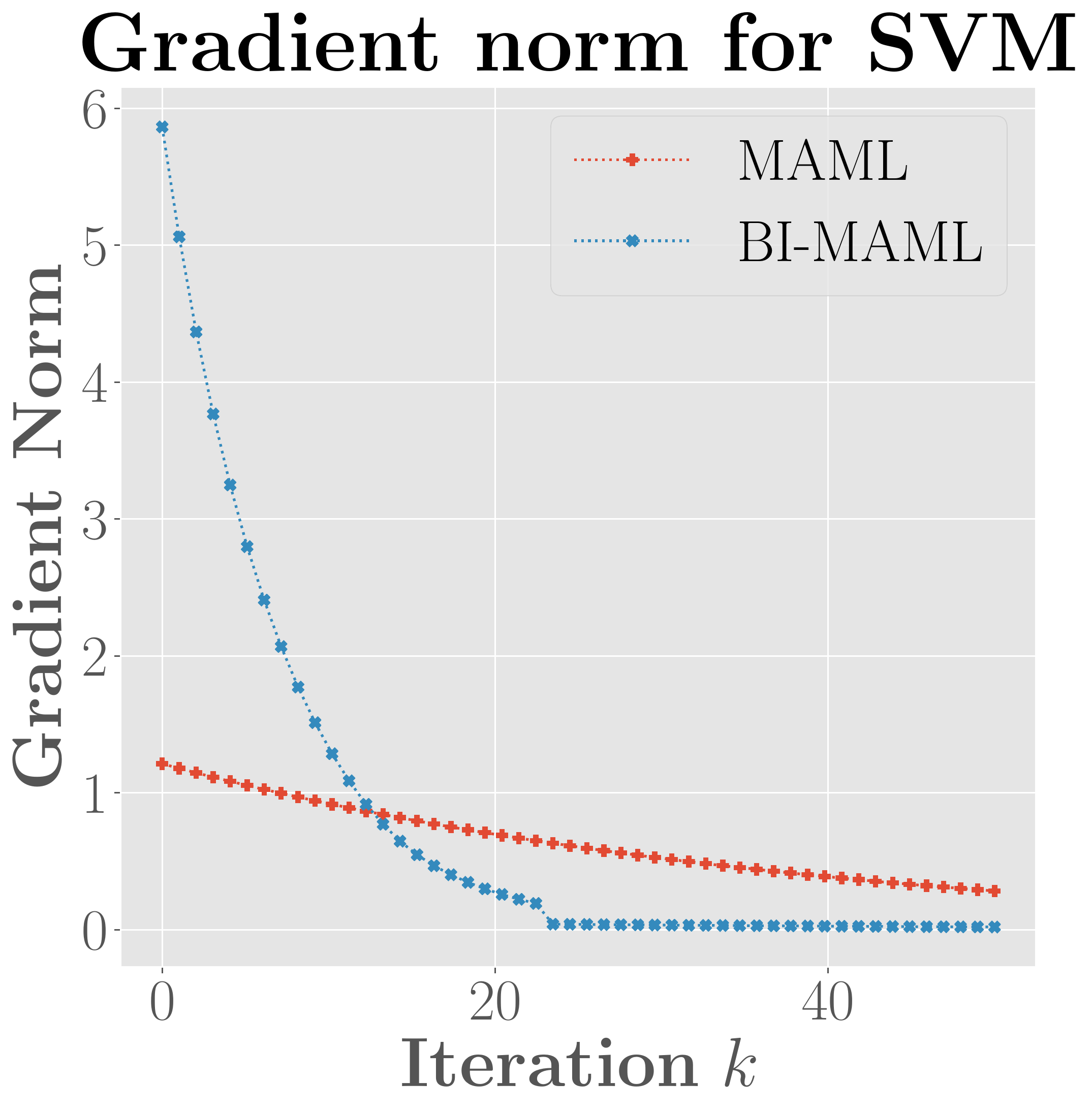}
  \caption{Binary class.~with \svm on synthetic data}
  \label{fig:svm_gradient}
\end{subfigure}\hfill
\begin{subfigure}[b]{.242\textwidth}
  \centering
  \includegraphics[width=\linewidth]{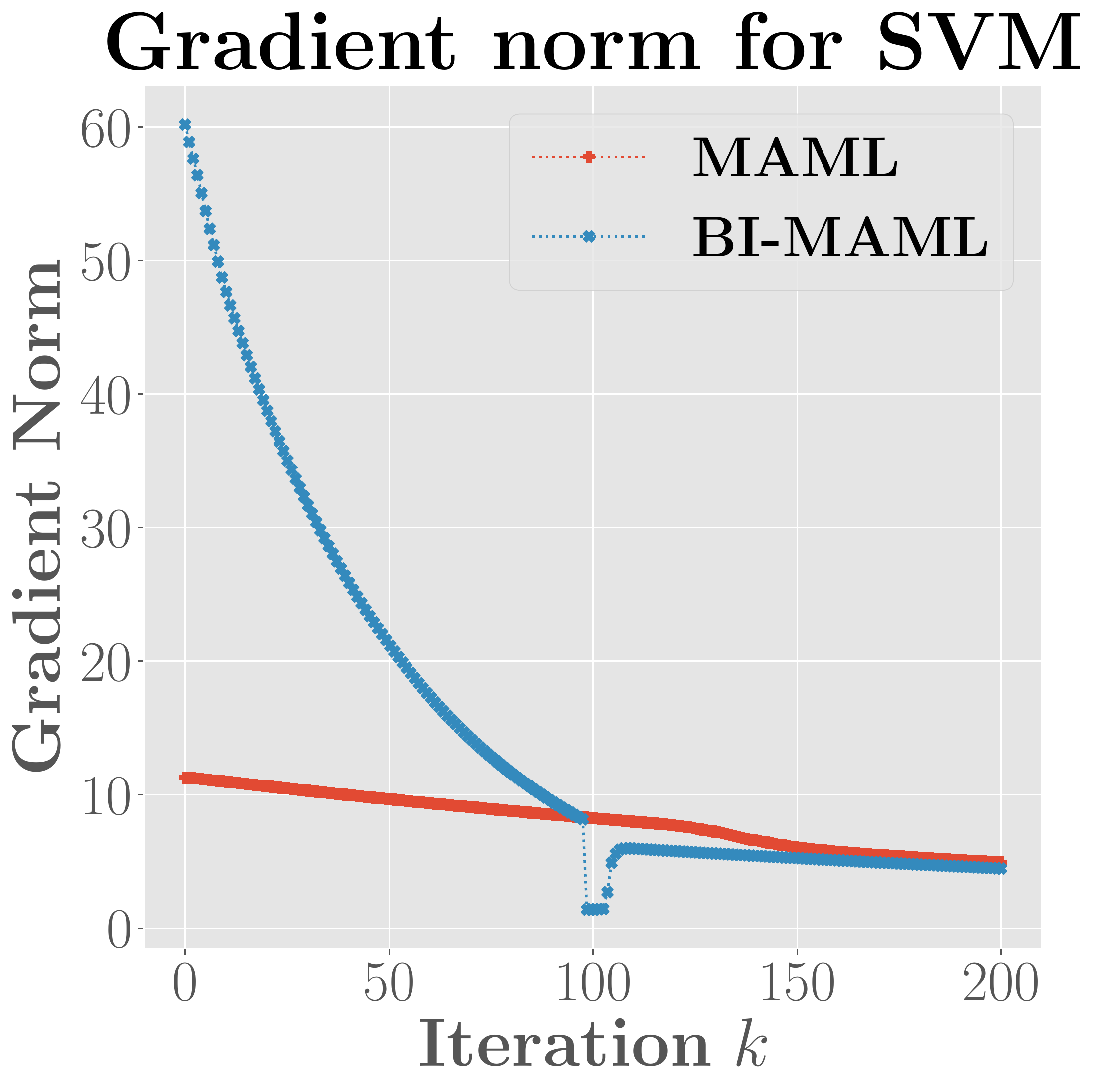}
  \caption{Binary classification with \svm on \texttt{MNIST}}
  \label{fig:svm-mnist-gradient}
\end{subfigure}
\caption{Comparisons on the performance of \bmm versus \maml on linear regression and binary classification with \svm. The figures on the first row compare the \bmm and \maml with function values on the \maml loss $F(w_k)$ at each iteration $k$. The figures in the second row show the trajectory of the corresponding gradient norm values evaluated at every iteration $k$. The blue lines represent the performance on \bmm, and the red lines represent that on \maml. }
\label{fig:images}
\end{figure}

It is noteworthy that even if our method only provably works for strongly convex loss functions $f_i$, it empirically achieves good performance on real world data for many strongly convex and even convex loss functions, such as \svm. 
We remark that the gradient norm for the \bmm shown in \cref{fig:lr_gradient,fig:svm_gradient,fig:svm-mnist-gradient,fig:lr-diabetes-gradient} represents different quantities in different stages of the algorithm. Specifically, it denotes $\|\nabla F(w_k)\|$ when the \bmm descends on $F$ at the \kth iteration; analogously, it represents $\|\nabla f(w_k)\|$ when the \bmm descends on $f$. The jumps in \cref{fig:svm_gradient,fig:svm-mnist-gradient,fig:svm-mnist-function} show clearly the transitions between two phases.

%% file: Meta_Learning_in_the_Continuous_Time_Limit.bbl
\begin{thebibliography}{43}
\providecommand{\natexlab}[1]{#1}
\providecommand{\url}[1]{\texttt{#1}}
\expandafter\ifx\csname urlstyle\endcsname\relax
  \providecommand{\doi}[1]{doi: #1}\else
  \providecommand{\doi}{doi: \begingroup \urlstyle{rm}\Url}\fi

\bibitem[Andrychowicz et~al.(2016)Andrychowicz, Denil, Gomez, Hoffman, Pfau,
  Schaul, Shillingford, and De~Freitas]{andrychowicz2016learning}
Marcin Andrychowicz, Misha Denil, Sergio Gomez, Matthew~W Hoffman, David Pfau,
  Tom Schaul, Brendan Shillingford, and Nando De~Freitas.
\newblock Learning to learn by gradient descent by gradient descent.
\newblock In \emph{Advances in neural information processing systems}, pages
  3981--3989, 2016.

\bibitem[Antoniou et~al.(2018)Antoniou, Edwards, and
  Storkey]{antoniou2018train}
Antreas Antoniou, Harrison Edwards, and Amos Storkey.
\newblock How to train your maml.
\newblock \emph{arXiv preprint arXiv:1810.09502}, 2018.

\bibitem[Behl et~al.(2019)Behl, Baydin, and Torr]{behl2019alpha}
Harkirat~Singh Behl, At{\i}l{\i}m~G{\"u}ne{\c{s}} Baydin, and Philip~HS Torr.
\newblock Alpha maml: Adaptive model-agnostic meta-learning.
\newblock In \emph{6th ICML Workshop on Automated Machine Learning}, 2019.

\bibitem[Duan et~al.(2016)Duan, Schulman, Chen, Bartlett, Sutskever, and
  Abbeel]{duan2016rl}
Yan Duan, John Schulman, Xi~Chen, Peter~L Bartlett, Ilya Sutskever, and Pieter
  Abbeel.
\newblock Rl$^2$: Fast reinforcement learning via slow reinforcement learning.
\newblock \emph{arXiv preprint arXiv:1611.02779}, 2016.

\bibitem[Duan et~al.(2017)Duan, Andrychowicz, Stadie, Ho, Schneider, Sutskever,
  Abbeel, and Zaremba]{duan2017one}
Yan Duan, Marcin Andrychowicz, Bradly Stadie, OpenAI~Jonathan Ho, Jonas
  Schneider, Ilya Sutskever, Pieter Abbeel, and Wojciech Zaremba.
\newblock One-shot imitation learning.
\newblock In \emph{Advances in neural information processing systems}, pages
  1087--1098, 2017.

\bibitem[Edwards and Storkey(2016)]{edwards2016towards}
Harrison Edwards and Amos Storkey.
\newblock Towards a neural statistician.
\newblock \emph{arXiv preprint arXiv:1606.02185}, 2016.

\bibitem[Fallah et~al.(2019)Fallah, Mokhtari, and
  Ozdaglar]{fallah2019convergence}
Alireza Fallah, Aryan Mokhtari, and Asuman Ozdaglar.
\newblock On the convergence theory of gradient-based model-agnostic
  meta-learning algorithms.
\newblock \emph{arXiv preprint arXiv:1908.10400}, 2019.

\bibitem[Finn et~al.(2017)Finn, Abbeel, and Levine]{Finn2017}
Chelsea Finn, Pieter Abbeel, and Sergey Levine.
\newblock {Model-agnostic meta-learning for fast adaptation of deep networks}.
\newblock \emph{34th International Conference on Machine Learning, ICML 2017},
  3:\penalty0 1856--1868, 2017.

\bibitem[Finn et~al.(2019)Finn, Rajeswaran, Kakade, and Levine]{finn2019online}
Chelsea Finn, Aravind Rajeswaran, Sham Kakade, and Sergey Levine.
\newblock Online meta-learning.
\newblock In \emph{ICML}, 2019.

\bibitem[Grant et~al.(2018)Grant, Finn, Levine, Darrell, and
  Griffiths]{grant2018recasting}
Erin Grant, Chelsea Finn, Sergey Levine, Trevor Darrell, and Thomas Griffiths.
\newblock Recasting gradient-based meta-learning as hierarchical bayes.
\newblock \emph{arXiv preprint arXiv:1801.08930}, 2018.

\bibitem[Helmke and Moore(2012)]{helmke2012optimization}
Uwe Helmke and John~B Moore.
\newblock \emph{Optimization and dynamical systems}.
\newblock Springer Science \& Business Media, 2012.

\bibitem[Hirsch et~al.(2012)Hirsch, Smale, and Devaney]{hirsch2012differential}
Morris~W Hirsch, Stephen Smale, and Robert~L Devaney.
\newblock \emph{Differential equations, dynamical systems, and an introduction
  to chaos}.
\newblock Academic press, 2012.

\bibitem[Hochreiter et~al.(2001)Hochreiter, Younger, and
  Conwell]{hochreiter2001learning}
Sepp Hochreiter, A~Steven Younger, and Peter~R Conwell.
\newblock Learning to learn using gradient descent.
\newblock In \emph{International Conference on Artificial Neural Networks},
  pages 87--94. Springer, 2001.

\bibitem[Hsu et~al.(2019)Hsu, Levine, and Finn]{hsu2018unsupervised}
Kyle Hsu, Sergey Levine, and Chelsea Finn.
\newblock Unsupervised learning via meta-learning.
\newblock In \emph{International Conference on Learning Representations}, 2019.

\bibitem[Krichene et~al.(2015)Krichene, Bayen, and Bartlett]{Krichene2015}
Walid Krichene, Alexandre~M. Bayen, and Peter~L. Bartlett.
\newblock {Accelerated mirror descent in continuous and discrete time}.
\newblock \emph{Advances in Neural Information Processing Systems},
  2015-Janua:\penalty0 2845--2853, 2015.
\newblock ISSN 10495258.

\bibitem[Lacoste et~al.(2017)Lacoste, Boquet, Rostamzadeh, Oreshkin, Chung, and
  Krueger]{lacoste2017deep}
Alexandre Lacoste, Thomas Boquet, Negar Rostamzadeh, Boris Oreshkin, Wonchang
  Chung, and David Krueger.
\newblock Deep prior.
\newblock \emph{arXiv preprint arXiv:1712.05016}, 2017.

\bibitem[Li et~al.(2018)Li, Yang, Song, and Hospedales]{li2018learning}
Da~Li, Yongxin Yang, Yi-Zhe Song, and Timothy~M Hospedales.
\newblock Learning to generalize: Meta-learning for domain generalization.
\newblock In \emph{Thirty-Second AAAI Conference on Artificial Intelligence},
  2018.

\bibitem[Lu(2020)]{lu2020s}
Haihao Lu.
\newblock An $ {O(s^r)} $-resolution {ODE} framework for discrete-time
  optimization algorithms and applications to convex-concave saddle-point
  problems.
\newblock \emph{arXiv preprint arXiv:2001.08826}, 2020.

\bibitem[Lyapunov(1992)]{lyapunov1992general}
Aleksandr~Mikhailovich Lyapunov.
\newblock The general problem of the stability of motion.
\newblock \emph{International journal of control}, 55\penalty0 (3):\penalty0
  531--534, 1992.

\bibitem[Mendonca et~al.(2019)Mendonca, Gupta, Kralev, Abbeel, Levine, and
  Finn]{mendonca2019guided}
Russell Mendonca, Abhishek Gupta, Rosen Kralev, Pieter Abbeel, Sergey Levine,
  and Chelsea Finn.
\newblock Guided meta-policy search.
\newblock In \emph{Advances in Neural Information Processing Systems}, pages
  9653--9664, 2019.

\bibitem[Muehlebach and Jordan(2019)]{muehlebach2019dynamical}
Michael Muehlebach and Michael~I Jordan.
\newblock A dynamical systems perspective on nesterov acceleration.
\newblock \emph{arXiv preprint arXiv:1905.07436}, 2019.

\bibitem[Munkhdalai and Yu(2017)]{munkhdalai2017meta}
Tsendsuren Munkhdalai and Hong Yu.
\newblock Meta networks.
\newblock In \emph{Proceedings of the 34th International Conference on Machine
  Learning-Volume 70}, pages 2554--2563. JMLR. org, 2017.

\bibitem[Naik and Mammone(1992)]{naik1992meta}
Devang~K Naik and Richard~J Mammone.
\newblock Meta-neural networks that learn by learning.
\newblock In \emph{[Proceedings 1992] IJCNN International Joint Conference on
  Neural Networks}, volume~1, pages 437--442. IEEE, 1992.

\bibitem[Pan and Yang(2009)]{pan2009survey}
Sinno~Jialin Pan and Qiang Yang.
\newblock A survey on transfer learning.
\newblock \emph{IEEE Transactions on knowledge and data engineering},
  22\penalty0 (10):\penalty0 1345--1359, 2009.

\bibitem[Radford et~al.(2019)Radford, Wu, Child, Luan, Amodei, and
  Sutskever]{radford2019language}
Alec Radford, Jeffrey Wu, Rewon Child, David Luan, Dario Amodei, and Ilya
  Sutskever.
\newblock Language models are unsupervised multitask learners.
\newblock \emph{OpenAI Blog}, 1\penalty0 (8):\penalty0 9, 2019.

\bibitem[Raghu et~al.(2019)Raghu, Raghu, Bengio, and Vinyals]{raghu2019rapid}
Aniruddh Raghu, Maithra Raghu, Samy Bengio, and Oriol Vinyals.
\newblock Rapid learning or feature reuse? towards understanding the
  effectiveness of maml.
\newblock \emph{arXiv preprint arXiv:1909.09157}, 2019.

\bibitem[Rajeswaran et~al.(2019)Rajeswaran, Finn, Kakade, and
  Levine]{rajeswaran2019meta}
Aravind Rajeswaran, Chelsea Finn, Sham~M Kakade, and Sergey Levine.
\newblock Meta-learning with implicit gradients.
\newblock In \emph{Advances in Neural Information Processing Systems}, pages
  113--124, 2019.

\bibitem[Ravi and Larochelle(2016)]{ravi2016optimization}
Sachin Ravi and Hugo Larochelle.
\newblock Optimization as a model for few-shot learning.
\newblock 2016.

\bibitem[Santoro et~al.(2016)Santoro, Bartunov, Botvinick, Wierstra, and
  Lillicrap]{santoro2016meta}
Adam Santoro, Sergey Bartunov, Matthew Botvinick, Daan Wierstra, and Timothy
  Lillicrap.
\newblock Meta-learning with memory-augmented neural networks.
\newblock In \emph{International conference on machine learning}, pages
  1842--1850, 2016.

\bibitem[Schmidhuber(1987)]{schmidhuber1987evolutionary}
J{\"u}rgen Schmidhuber.
\newblock \emph{Evolutionary principles in self-referential learning, or on
  learning how to learn: the meta-meta-... hook}.
\newblock PhD thesis, Technische Universit{\"a}t M{\"u}nchen, 1987.

\bibitem[Schropp and Singer(2000)]{schropp2000dynamical}
Johannes Schropp and I~Singer.
\newblock A dynamical systems approach to constrained minimization.
\newblock \emph{Numerical functional analysis and optimization}, 21\penalty0
  (3-4):\penalty0 537--551, 2000.

\bibitem[Schweighofer and Doya(2003)]{schweighofer2003meta}
Nicolas Schweighofer and Kenji Doya.
\newblock Meta-learning in reinforcement learning.
\newblock \emph{Neural Networks}, 16\penalty0 (1):\penalty0 5--9, 2003.

\bibitem[Shi et~al.(2019)Shi, Du, Su, and Jordan]{shi2019acceleration}
Bin Shi, Simon~S Du, Weijie Su, and Michael~I Jordan.
\newblock Acceleration via symplectic discretization of high-resolution
  differential equations.
\newblock In \emph{Advances in Neural Information Processing Systems}, pages
  5745--5753, 2019.

\bibitem[Snell et~al.(2017)Snell, Swersky, and Zemel]{snell2017prototypical}
Jake Snell, Kevin Swersky, and Richard Zemel.
\newblock Prototypical networks for few-shot learning.
\newblock In \emph{Advances in neural information processing systems}, pages
  4077--4087, 2017.

\bibitem[Song et~al.(2019)Song, Gao, Yang, Choromanski, Pacchiano, and
  Tang]{song2019maml}
Xingyou Song, Wenbo Gao, Yuxiang Yang, Krzysztof Choromanski, Aldo Pacchiano,
  and Yunhao Tang.
\newblock Es-maml: Simple hessian-free meta learning.
\newblock \emph{arXiv preprint arXiv:1910.01215}, 2019.

\bibitem[Su et~al.(2014)Su, Boyd, and Candes]{su2014differential}
Weijie Su, Stephen Boyd, and Emmanuel Candes.
\newblock A differential equation for modeling nesterov's accelerated gradient
  method: Theory and insights.
\newblock In \emph{Advances in Neural Information Processing Systems}, pages
  2510--2518, 2014.

\bibitem[Thrun and Pratt(1998)]{thrun2012learning}
Sebastian Thrun and Lorien Pratt.
\newblock \emph{Learning to learn}.
\newblock Springer Science \& Business Media, 1998.

\bibitem[Vinyals et~al.(2016)Vinyals, Blundell, Lillicrap, Wierstra,
  et~al.]{vinyals2016matching}
Oriol Vinyals, Charles Blundell, Timothy Lillicrap, Daan Wierstra, et~al.
\newblock Matching networks for one shot learning.
\newblock In \emph{Advances in neural information processing systems}, pages
  3630--3638, 2016.

\bibitem[Wang et~al.(2016)Wang, Kurth-Nelson, Tirumala, Soyer, Leibo, Munos,
  Blundell, Kumaran, and Botvinick]{wang2016learning}
Jane~X Wang, Zeb Kurth-Nelson, Dhruva Tirumala, Hubert Soyer, Joel~Z Leibo,
  Remi Munos, Charles Blundell, Dharshan Kumaran, and Matt Botvinick.
\newblock Learning to reinforcement learn.
\newblock \emph{arXiv preprint arXiv:1611.05763}, 2016.

\bibitem[Wang and Hebert(2016)]{wang2016learning1}
Yu-Xiong Wang and Martial Hebert.
\newblock Learning to learn: Model regression networks for easy small sample
  learning.
\newblock In \emph{European Conference on Computer Vision}, pages 616--634.
  Springer, 2016.

\bibitem[Wilson et~al.(2016)Wilson, Recht, and Jordan]{wilson2016lyapunov}
Ashia~C Wilson, Benjamin Recht, and Michael~I Jordan.
\newblock A lyapunov analysis of momentum methods in optimization.
\newblock \emph{arXiv preprint arXiv:1611.02635}, 2016.

\bibitem[Yang et~al.(2019)Yang, Caluwaerts, Iscen, Tan, and
  Finn]{yang2019norml}
Yuxiang Yang, Ken Caluwaerts, Atil Iscen, Jie Tan, and Chelsea Finn.
\newblock Norml: No-reward meta learning.
\newblock In \emph{Proceedings of the 18th International Conference on
  Autonomous Agents and MultiAgent Systems}, pages 323--331. International
  Foundation for Autonomous Agents and Multiagent Systems, 2019.

\bibitem[Yu et~al.(2018)Yu, Finn, Xie, Dasari, Zhang, Abbeel, and
  Levine]{yu2018one}
Tianhe Yu, Chelsea Finn, Annie Xie, Sudeep Dasari, Tianhao Zhang, Pieter
  Abbeel, and Sergey Levine.
\newblock One-shot imitation from observing humans via domain-adaptive
  meta-learning.
\newblock In \emph{ICLR workshop}, 2018.

\end{thebibliography}
